\documentclass{article}

\usepackage[top=1in, bottom=1in, left=1in, right=1in]{geometry}

\usepackage{microtype}
\usepackage{graphicx}
\usepackage{subfig}
\usepackage[numbers]{natbib}
\usepackage{booktabs} % for professional tables

\usepackage{hyperref}
\usepackage{color}

\usepackage[utf8]{inputenc} % allow utf-8 input
\usepackage[T1]{fontenc}    % use 8-bit T1 fonts
\usepackage{hyperref}       % hyperlinks
\usepackage{url}            % simple URL typesetting
\usepackage{booktabs}       % professional-quality tables
\usepackage{amsfonts}       % blackboard math symbols
\usepackage{nicefrac}       % compact symbols for 1/2, etc.
\usepackage{microtype}      % microtypography
\usepackage{amsmath}
\usepackage{graphicx}
\usepackage{multicol}
\usepackage{multirow}
\usepackage{caption}
\usepackage{algorithmic}
\usepackage{enumitem}
\usepackage{url}
\usepackage{tcolorbox}

\newtheorem{theorem}{Theorem}[section]

\newtheorem{claim}[theorem]{Claim}
\newtheorem{remark}[theorem]{Remark}
\newtheorem{example}[theorem]{Example}

\newtheorem{lemma}[theorem]{Lemma}
\newtheorem{assumption}{Assumption}
\newtheorem{discussion}[theorem]{Discussion}

\newtheorem{definition}[theorem]{Definition}

\newtheorem{observation}[theorem]{Observation}
\newtheorem{problem}{Problem}

\newcommand{\eps}{\varepsilon}
\renewcommand{\epsilon}{\varepsilon}

\newcommand{\eat}[1]{}

\newcommand{\I}{\mathbf{I}}
\newcommand{\R}{\mathbb{R}}

\renewcommand{\eqref}[1]{\ref{#1}}

\usepackage{tikz}

\newcommand{\specialcell}[2][c]{%
  \begin{tabular}[#1]{@{}c@{}}#2\end{tabular}}

\newcommand{\calD}{{\mathcal{D}}}
\newcommand{\calF}{{\mathcal{F}}}
\newcommand{\calG}{{\mathcal{G}}}

\newcommand{\calQ}{{\mathcal{Q}}}

\newcommand{\calX}{{\mathcal{X}}}

\providecommand{\set}[1]{{\{#1\}}}

\makeatletter
\renewcommand*{\@opargbegintheorem}[3]{\trivlist
      \item[\hskip \labelsep{\bfseries #1\ #2}] \textbf{(#3)}\ \itshape}
\makeatother

\allowdisplaybreaks

\title{\bf Fair Classification with Noisy Protected Attributes: \\ A Framework with Provable Guarantees}
\author{L. Elisa Celis \\ Yale University \and Lingxiao Huang \\ Huawei \and Vijay Keswani \\ Yale University \and Nisheeth K. Vishnoi \\ Yale University}

\begin{document}

\maketitle
	\begin{abstract}
		We present an optimization framework for learning a fair classifier in the presence of noisy perturbations in the protected attributes.
        Compared to prior work,	our framework can be employed with a very general class of linear and {\em linear-fractional}  fairness constraints, can handle multiple,  {\em non-binary} protected attributes, and outputs a  classifier that comes with provable guarantees on {\em both}  accuracy and fairness.
        Empirically, we show that our framework can be used to attain either statistical rate or false positive rate fairness guarantees with a minimal loss in accuracy, even when the noise is large, in two real-world datasets. 
	\end{abstract}
	
\clearpage
\tableofcontents
\clearpage

\section{Introduction}
	\label{sec:introduction}

	Fair classification has been a topic of intense study  due to the growing importance of addressing social biases in automated prediction.
	Consequently, a host of fair classification algorithms have been proposed that learn from data \cite{bellamy2018ai,zafar2017fairness,zhang2018mitigating,menon2018the,goel2018non,celis2019classification,hardt2016equality,menon2018the,fish2016confidence,goh2016satisfying,pleiss2017on, woodworth2017learning,dwork2018decoupled}.

    Fair classifiers 
    need metrics that capture the extent of similarity in performance for different groups.
    The performance of a classifier $f$ for  group $z$ can be defined in many ways and a general definition that captures the most common group performance metrics in literature is the following: given events $\xi$ and $\xi'$ (that depend on classifier and/or class label $Y$), the performance for group $z$ can be quantified as $\Pr[\xi \mid \xi', z]$ (Defn~\ref{def:performance}).
    For instance, we get group-specific statistical rate (a linear metric) by setting $\xi := (f{=}1)$ and $\xi' := \emptyset$ and group-specific false discovery rate (a ``linear-fractional metric'' \cite{celis2019classification}) by setting $\xi := (Y{=}0)$ and $\xi' := (f{=}1)$.
    Then, given a performance function, fairness constraints in classification impose a feasible classifier to have similar performance for all groups,
    by constraining the performance difference to be within $\tau$ of each other either multiplicatively or additively.
    For any fixed group performance function, multiplicative constraints imply additive constraints and, hence, traditionally studied \cite{calders2010three,zafar2017fair,zafar2017fairness,menon2018cost,celis2019classification} 
    (see also Remark~\ref{remark:add_multi}).

    {The choice of fairness metric depends on the context and application.
    For instance, in a lending setting, statistical rate metric can capture the disparity in loan approval rate across gender/race \cite{fairlending2010}.
    In a recidivism assessment setting, false positive rate metric is more relevant as it captures the disparity in proportion to defendants falsely assigned high-risk across racial groups \cite{angwin2016machine}.
    In other settings, e.g., healthcare, where the costs associated with positive classification are large, false discovery rate is alternately employed to assess the disparity in proportion to the treated patients who didn't require treatment across protected attribute types \cite{srivastava2019mathematical}.
    }
	
    Most of the aforementioned fair classification algorithms crucially assume that one has access to the protected attributes (e.g., race, gender) for training and/or deployment.
	Data collection, however, is a complex process and may contain recording and reporting errors, unintentional or otherwise~\cite{saez2013tackling}.
	Cleaning the data also requires making difficult and political decisions along the way, yet is often necessary especially when it comes to questions of race, gender, or identity \cite{melissa2000shades}.
    {Further, information about protected attributes may be missing entirely \cite{Data2004EliminatingHD}, or legally prohibited from being used directly, as in the case of lending applications for non-mortgage products in US \cite{federal1993closing}.}
	In such cases, protected attributes can be predicted from other data, however, we know that this process is itself contains errors and biases \cite{muthukumar2018understanding,buolamwini2018gender}. 
    The above scenarios raise a  challenge for existing fair classifiers as they may not achieve the same fairness as they would if the data were perfect. 
    This raises the question of learning fair classifiers in the presence of noisy protected attributes, and has attracted recent attention \cite{awasthi2020equalized,lamy2019noise,wang2020robust}. 

    \subsection{Our contributions} We study the setting of ``flipping noises'' where a protected type $Z=i$ may be flipped to $\hat{Z} = j$ with some known fixed probability $H_{ij}$ (Definition~\ref{def:flippingnoise}).
	We present an optimization framework for learning a fair classifier that can  handle: 
    1)	\textbf{flipping noises} in the train, test, and future samples,
	2) multiple, {\bf non-binary} protected attributes, and
	{
	3)	{multiple} fairness metrics, including the general class of {\bf linear-fractional} metrics (e.g., statistical parity, false discovery rate) in a multiplicative sense.
	}
	Our framework can learn a near-optimal fair classifier on the underlying dataset with high probability and comes with provable guarantees on both accuracy and fairness.
	
	We implement our framework using the logistic loss function \cite{freedman2009statistical} and examine it on 
	\textbf{Adult} and \textbf{COMPAS} datasets (Section~\ref{sec:empirical}).
	We consider sex and race as the protected attribute and generate noisy datasets varying flipping noise parameters. 
	For \textbf{COMPAS} dataset, the race protected attribute is non-binary.
    {We use statistical rate, false positive rate, and false discovery rate fairness metrics}, and compare against natural baselines and  existing noise-tolerant fair classification algorithms \cite{lamy2019noise, awasthi2020equalized, wang2020robust}.
	The empirical results show that, for most combinations of dataset and protected attribute (both binary and non-binary), our framework attains better fairness than an unconstrained classifier, with a minimal loss in accuracy.
    Further, in most cases, the fairness-accuracy tradeoff of our framework, for statistical and false positive rate, is also better than the baselines and other noise-tolerant fair classification algorithms,
    which either do not always achieve high fairness levels or suffer a larger loss in accuracy for achieving high fairness levels compared to our framework 
    (Table~\ref{tab:sr_results}).
	For false discovery rate (linear-fractional metric), our approach  has better fairness-accuracy tradeoff than baselines for \textbf{Adult} dataset and similar tradeoff as the best-performing baseline for \textbf{COMPAS} dataset (Table~\ref{tab:fdr_results}).

    \noindent
    \subsection{Techniques}	Our framework starts by designing \textit{denoised} constraints to achieve the desired fairness guarantees which take into account the noise in the protected attribute (Program \eqref{eq:progdenoised}).
	{The desired fairness is governed using an input parameter $\tau\in [0,1]$.}
	The key is to estimate each group-specific performance on the underlying dataset, which enables us to handle non-binary protected attributes.
	Concretely, we represent a group-specific performance as a ratio and estimate its numerator and denominator separately, which enables us to handle linear-fractional constraints.
    Subsequently, we show that an optimizer $f^\Delta$ of our program is provably {\bf both} approximately optimal and fair on the underlying dataset (Theorem~\ref{thm:denoised}) with high probability under a mild assumption
	that an optimizer $f^\star$ of the underlying program (Program \eqref{eq:progtarget}) has a non-trivial lower bound on the group-specific prediction rate (Assumption~\ref{assumption:ratio}).

	The constraints in our program enable us to capture the range of alteration in the probability of any classifier prediction for different protected attribute types due to flipping noises and, consequently, allow us to provide  guarantees on $f^\Delta$ (Theorem~\ref{thm:denoised}).
	The guarantee on accuracy uses the fact that an optimal fair classifier $f^\star$ for the underlying uncorrupted dataset is \textit{likely} to be feasible for Program \eqref{eq:progdenoised} as well, which ensures that the empirical risk of $f^\Delta$ is less than $f^\star$ (Lemma~\ref{lm:denoised1}).
    The guarantee on the fairness of $f^\Delta$ is attained by arguing that classifiers that considerably violate the desired fairness guarantee are infeasible for Program~\eqref{eq:progdenoised} with high probability (Lemma~\ref{lm:denoised2}).
    The key technical idea is to discretize the  space of unfair classifiers by carefully chosen multiple $\eps$-nets with different violation degrees to our denoised program, and upper bound the capacity of the union of all nets via a VC-dimension bound.

    \subsection{Related work}	
    \paragraph{Noise-tolerant fair classification. }
    \cite{lamy2019noise} consider binary protected attributes, linear fairness metrics including statistical rate (SR) and equalized odds constraints \cite{donini2018empirical}.
    They give a provable algorithm that achieves an approximate optimal fair classifier by down-scaling the ``fairness tolerance'' parameter in the constraints to adjust for the noise.
    In contrast, our approach estimates the altered form of fairness metrics in the noisy setting, and hence, can also handle linear-fractional metrics and non-binary attributes.
    \citet{awasthi2020equalized} study the performance of the equalized odds post-processing method of \citet{hardt2016equality} for a single noisy binary protected attribute. 
    {However, their analysis assumes that the protected attributes of test/future samples are uncorrupted.
	Our framework, instead, can handle multiple, non-binary attributes and noise in test/future samples.}
	\citet{wang2020robust} propose a robust optimization approach to solve the noisy fair classification problem.
	By proposing an iterative procedure to solve the arising min-max problem, they can only guarantee a stochastic classifier that is near-optimal w.r.t. accuracy and near-feasible w.r.t. fairness constraints on the underlying dataset {\em in expectation}, but not with high probability.
    Moreover, their iterative procedure relies on a minimization oracle, which is not always computationally tractable and their  practical algorithm does not share the guarantees of their theoretical algorithm for the output classifier. 
    In contrast, our denoised fairness program ensures that the optimal classifier is deterministic, near-optimal w.r.t. {\bf both} accuracy and near-feasible w.r.t. fairness constraints on the underlying dataset, with high probability.
    Additionally, we define performance disparity across protected attribute values as the ratio of the ``performance'' for worst and best-performing groups (multiplicative constraints), while existing works~\cite{lamy2019noise,awasthi2020equalized,wang2020robust} define the disparity using the additive difference across the protected attribute values (additive constraints); see Remark~\ref{rem:metric_comparison}.

	\paragraph{Fair classification.}  
	Many works have focused on formulating fair classification problems as  constrained optimization problems, 
	\cite{zafar2017fairness,zhang2018mitigating,menon2018the,goel2018non,celis2019classification},
	\cite{hardt2016equality,zafar2017fair,menon2018the,celis2019classification}, and developing algorithms for it.
	Another class of algorithms first learn an unconstrained optimal classifier and then shift the decision boundary according to the fairness requirement, e.g.,~\cite{fish2016confidence,hardt2016equality,goh2016satisfying,pleiss2017on, woodworth2017learning,dwork2018decoupled}.
	In contrast to our work, the assumption in all of these approaches is that the algorithm is given perfect information about the protected class.

	\noindent
	\paragraph{Data correction.} 
	Cleaning raw data is a significant step in the pipeline, and efforts to correct for missing or inaccurately coded attributes have been studied in-depth for protected attributes, e.g., in the context of the census \cite{melissa2000shades}.
	An alternate approach considers changing the composition of the dataset itself to correct for known biases in representation
	\cite{calders2009building,kamiran2009classifying,kamiran2012data}, 
	\cite{del2018obtaining,wang2019repairing}, 
	\cite{calmon2017optimized, celisdata}.
	In either case, the correction process, while important, can be imperfect and our work can help by starting with these improved yet imperfect datasets in order to build fair classifiers. 
	
	\noindent
	\paragraph{Unknown protected attributes.} 
	{	A related setting  is  when the information of some protected attributes is unknown.
	\cite{gupta2018proxy,chen2019fairness,kallus2020assessing} considered this setting of unknown protected attributes and designed algorithms to improve fairness or assess disparity.
	In contrast, our approach aims to derive necessary information from the observed protected attributes to design alternate fairness constraints using the noisy attribute.
	}
	
	\noindent
	\paragraph{Classifiers robust to the choice of datasets.}
	{
    \cite{friedler2018comparative} observed that fair classification algorithms may not be stable with respect to variations in the training dataset.
	\cite{hashimoto2018fairness} proved that empirical risk minimization amplifies representation disparity over time.
	}
	Towards this, certain variance reduction or stability techniques have been introduced; see e.g., 
	\cite{huang2019stable}.
	However, their approach cannot be used to learn a classifier that is provably fair over the underlying dataset.

	\noindent
	\paragraph{Noise in labels.}
	\citet{blum2020recovering,Biswas2020EnsuringFU}  study fair classification when the label in the input dataset is noisy. 
	The main difference of these from our work is that they consider noisy {\em labels} instead of noisy protected attributes, which makes our denoised algorithms very different 
    {since the accuracy of protected attributes mainly relates to the fairness of the classifier but the accuracy of labels primarily affect the empirical loss}. 

\section{The model}
	\label{sec:problem}
	
	Let $\calD=\calX\times [p]\times\left\{0,1\right\}$ denote the underlying domain ($p\geq 2$ is an integer).
	Each sample $(X,Z,Y)$ drawn from $\calD$ contains a protected attributes $Z$, a class label $Y\in \left\{0,1\right\}$, and non-protected features $X \in \calX$.
	Here, we discuss a single protected attribute,  $Z=[p]$, and generalize our model and results to multiple protected attributes in Section~\ref{sec:complete_generalization}.
	We  assume that $X$ is a $d$-dimensional vector, for a given $d \in \mathbb{N}$, i.e.,  $\calX\subseteq \R^d$.
	Let $S=\left\{s_a=\left(x_a,z_a,y_a\right)\in \calD\right\}_{a\in [N]}$ be the (underlying, uncorrupted) dataset.
	Let $\calF\subseteq \left\{0,1\right\}^\calX$ denote a family of all possible allowed classifiers. 
	Given a loss function $L : \calF \times \calD \rightarrow \R_{\geq 0}$, the goal of unconstrained classification is to find a classifier $f\in \calF$ that minimizes \emph{the empirical risk} $\frac{1}{N}\sum_{a\in [N]} L(f, s_a)$.

	\noindent
	\paragraph{Fair classification and fairness metrics.}
	We consider the problem of classification for a general class of fairness metrics.
	Let $D$ denote the empirical distribution over $S$, i.e., selecting each sample $s_a$ with probability $1/N$.

	\begin{definition}[\bf{Linear/linear-fractional group performance functions~\cite{celis2019classification}}]
		\label{def:performance}
		Given a classifier $f\in \calF$ and $i\in [p]$, we call $q_i(f)$ the group performance of $Z=i$ if $\textstyle{q_i(f)=\Pr_D\left[\xi(f)\mid \xi'(f), Z=i\right]}$ for some events $\xi(f), \xi'(f)$ that might depend on the choice of $f$.
		If $\xi'$ does not depend on the choice of $f$, $q$ is said to be \textbf{linear}; otherwise, $q$ is said to be \textbf{linear-fractional}.
	\end{definition}
	
	\noindent
	At a high level, a classifier $f$ is considered to be fair w.r.t. $q$ if $q_1(f)\approx \cdots \approx q_p(f)$.
	Definition~\ref{def:performance} is general and contains many fairness metrics.
	For instance, if $\xi := (f=1)$ and $\xi':= \emptyset$, we have $q_i(f)= \Pr_D\left[f=1\mid Z=i\right]$ which is linear and called the statistical rate.
	If $\xi:=(Y=0)$ and $\xi':= (f=1)$, we have $q_i(f)= \Pr_D\left[Y=0\mid f=1, Z=i \right]$ which is linear-fractional and called the false discovery rate.
	See Table 1 in~\cite{celis2019classification} for a comprehensive set of special cases.
	Given a group performance function $q$, we define $\Omega_q: \calF\times \calD^\star \rightarrow [0,1]$ to be
	$
	\textstyle \Omega_q (f, S) := \min_{i\in [p]} q_i(f)/ \max_{i\in [p]} q_i(f)
	$
	as a specific fairness metric.
	Then we define the following fair classification problem: Given a group performance functions $q$ and a threshold $\tau\in [0,1]$, the goal is to learn an (approximate) optimal fair classifier $f\in \calF$ of the following program:
			\begin{equation} \tag{\bf TargetFair}
			\label{eq:progtarget}
			\begin{split}
			& \textstyle{\min_{f\in \calF} \frac{1}{N}\sum_{a\in [N]} L(f, s_a) \quad s.t.} \\
			& ~ \textstyle{\Omega_{q}(f, S)\geq \tau.}
			\end{split}
			\end{equation}
	\noindent
	For instance, we can set $q$ to be the statistical rate and $\tau = 0.8$ to encode the 80\% disparate impact rule \cite{biddle2006adverse}.
	Note that $\Omega_q(f, S)\geq 0.8$ is usually non-convex for certain $q$.
    Often, one considers a convex function as an estimate of $\Omega_q(f, S)$, for instance $\Omega_q(f, S)$ is formulated as a covariance-type function 
	in \cite{zafar2017fairness}, and as the weighted sum of the logs of the empirical estimate of favorable bias in \cite{goel2018non}.

\begin{remark}[\bf{Multiplicative v.s. additive fairness constraints}]
\label{remark:add_multi}
We note that	the fairness constraints mentioned above ($\Omega_q$) are \textbf{multiplicative} and appear in \cite{calders2010three,zafar2017fair,zafar2017fairness,menon2018cost,celis2019classification}.
Multiplicative fairness constraints control disparity across protected attribute values by ensuring that the ratio of the ``performance'' for the worst and best-performing groups are close.
	In contrast, related prior work for noisy fair classification \cite{lamy2019noise,awasthi2020equalized,wang2020robust} usually consider \textbf{additive} fairness constraints, i.e., of the form $\Omega'_q(f,S):= \max_{i\in [p]} q_i(f) - \min_{i\in [p]} q_i(f)\leq \tau'$ for some $\tau'\in [0,1]$ (difference instead of ratio).
	Specifically, letting $\tau' = 0$ in the additive constraint is equivalent to letting $\tau = 1$ in the multiplicative constraint with respect to the same group performance function $q$.
Note that \textbf{multiplicative implies additive}, i.e., given $\tau\in [0,1]$, we have that
	$\Omega(f,S)\geq \tau$ implies that $\Omega'(f,S)\leq 1-\tau$. 
However, the converse is not true: for instance, given arbitrary small $\tau' > 0$, we may learn a classifier $f^\star$ under additive constraints such that $\min_{i\in [p]} q_i(f^\star) = 0$ and $\max_{i\in [p]} q_i(f^\star) = \tau'$; however, such $f^\star$ violates the 80\% rule \cite{biddle2006adverse} that is equivalent to $\Omega_q(f,S)\geq 0.8$.
\end{remark}
	
\noindent
\paragraph{Noise model.} If $S$ is observed, we can directly use Program~\eqref{eq:progtarget}.
	However, as discussed earlier, the protected attributes in $S$ may be imperfect and we may only observe a noisy dataset $\widehat{S}$ instead of $S$.
	We consider the following noise model on the protected attributes  \cite{lamy2019noise,awasthi2020equalized,wang2020robust}.
	\begin{definition}[\bf{Flipping noises}]
	\label{def:flippingnoise}
		Let $H\in [0,1]^{p\times p}$ be a stochastic matrix with $\sum_{j\in [p]}H_{ij}=1$ and $H_{ii}>0.5$ for any $i\in [p]$.
		Assume each protected attribute $Z=i$ ($i\in [p]$) is observed as 
		$\widehat{Z} = j$ with probability $H_{ij}$, for any $j\in [p]$.
	\end{definition}

\eat{
	\begin{definition}[Flipping noises for binary $Z$]
		\label{def:flippingnoise}
		Suppose $p=2$, i.e., $Z\in \left\{0,1\right\}$.
		Let $\eta_0, \eta_1\in (0,0.5)$ be noise parameters.
		For each $a\in [N]$, we assume that the $a$th \emph{noisy} sample $\widehat{s}_a = (x_a, \widehat{z}_a, y_a)$ is realized as follows:
		\begin{itemize}
			\item If $z_a=0$, then $\widehat{z}_a=0$ with probability $1-\eta_0$ and $\widehat{z}_a=1$ with probability $\eta_0$.
			\item If $z_a=1$, then $\widehat{z}_a=0$ with probability $\eta_1$ and $\widehat{z}_a=1$ with probability $1-\eta_1$.
		\end{itemize}
	\end{definition}
	
}

	\noindent
	Note that $H$ can be \textbf{non-symmetric}.
	The assumption that $H_{ii}> 0.5$ ensures that the total flipping probability of each protected attribute is strictly less than a half. 
	Consequently, $H$ is a {diagonally-dominant} matrix, which is always non-singular \cite{horn2012matrix}.
	%	%
    {Due to noise, directly applying the same fairness constraints on $\widehat{S}$ may introduce bias on $S$ and, hence, modifications to the constraints are necessary; see Section~\ref{sec:discussion} for a discussion. 
	}
	
	\begin{remark}[\bf{Limitation of Definition~\ref{def:flippingnoise}}]
		In practice, we may not know $H$ explicitly, and can only estimate them by, say, finding a small appropriate sample of the data for which ground truth is known (or can be found), and computing estimates for $H$ accordingly \cite{kallus2020assessing}.
		{For instance, $H$ could be inferred from prior data that contains both true and noisy (or \textit{proxy}) protected attribute values; e.g., existing methods, such as Bayesian Improved Surname Geocoding method \cite{elliott2009using}, employ census data to construct conditional race membership probability models given surname and location.}
		In the following sections, we assume $H$ is given.
		For settings in which the estimates of $H$ may not be accurate, we analyze the influences of the estimation errors at the end of Section~\ref{sec:thm}. 
	\end{remark}

	\begin{problem}[\bf{Fair classification with noisy protected attributes}]
		\label{problem:general}
		Given a group performance functions $q$, a threshold $\tau\in [0,1]$, and a noisy dataset $\widehat{S}$ with noise matrix $H$, the goal is to learn an (approximate) optimal fair classifier $f\in \calF$ of Program~\eqref{eq:progtarget}.
	\end{problem}
	
\eat{	
	\begin{problem}[\bf{Fair classification with  noisy protected attributes}]
		\label{problem:simple}
		Given a binary protected attribute ($p=2$), a fairness constraint of the form $\gamma(f,S)\geq \tau$, a noisy dataset $\widehat{S}$ drawn from the underlying dataset $S$ with flipping noise parameters $\eta_0, \eta_1\in (0,0.5)$, and $\lambda \in (0,0.5)$ for which Assumption \ref{assumption:ratio} holds, the goal is to learn an (approximately) optimal fair classifier $f\in \calF$ of Program~\eqref{eq:progtarget}.
	\end{problem}
}

\section{Framework and theoretical results}
	\label{sec:denoised}
    We show how to learn an approximately fair classifier w.h.p. for Problem~\ref{problem:general} (Theorem~\ref{thm:denoised}).
	This result is generalized to multiple protected attributes/fairness metrics in Section~\ref{sec:complete_generalization}.
	The approach is to design \emph{denoised fairness constraints} over $\widehat{S}$ (Definition~\eqref{def:denoised}) that estimate the underlying constraints of Program~\eqref{eq:progtarget}, and solve the  constrained optimization problem (Program~\eqref{eq:progdenoised}).
	Let $f^\star\in \calF$ denote an optimal classifier of Program~(\eqref{eq:progtarget}).
	Our result relies on a natural assumption on $f^\star$.

	\begin{assumption}
		\label{assumption:ratio}
		There exists a constant $\lambda\in (0, 0.5)$ such that $\textstyle{\min_{i\in [p]}\Pr_D\left[\xi(f^\star), \xi'(f^\star), Z=i \right]\geq \lambda}$.
	\end{assumption}

	\noindent 
	Note that $\lambda$ is a lower bound for $\min_{i\in [p]}q_i(f^\star)$.
	In many applications we expect this assumption to hold.  
	For instance, $\lambda\geq 0.1$ if there are at least 20\% of samples with $Z=i$ and $\Pr_D\left[\xi(f^\star), \xi'(f^\star), Z=i\right]\geq 0.5$ for each $i\in [p]$.
	In practice, exact $\lambda$ is unknown but we can set $\lambda$ according to the context.
	This assumption is not strictly necessary, i.e., we can simply set $\lambda = 0$, but the scale of $\lambda$ decides certain capacity of classifiers that we do not want to learn, which affects the performance of our approaches; see Remark~\ref{remark:2}.

	\subsection{Our optimization framework}
	\label{sec:program}
	
    Let $\widehat{D}$ denote the empirical distribution over $\widehat{S}$.
	Let  $\widehat{u}(f) := \left(\Pr\left[\xi(f), \xi'(f), \widehat{Z}=i\right]\right)_{i\in [p]} $ and $	\textstyle \widehat{w}(f) := \left(\Pr\left[\xi'(f), \widehat{Z}=i\right]\right)_{i\in [p]}$.
    If $D$ and $\widehat{D}$ are clear from the context, we denote $\Pr_{D,\widehat{D}}\left[\cdot\right]$ by $\Pr\left[\cdot\right]$.
	Let $M:= \max_{i\in [p]} \|(H^\top)^{-1}_i\|_1$. 
	% denote the maximum $\ell_1$-norm of rows of $(H^\top)^{-1}$.
	%
	Define the denoised fairness constraints and the induced program as follows.

	\begin{definition}[\bf{Denoised fairness constraints}]
		\label{def:denoised}
		Given a classifier $f\in \calF$, for $i\in [p]$ let
		$
        \textstyle	\Gamma_i(f):= \frac{(H^\top)^{-1}_i \widehat{u}(f)}{(H^\top)^{-1}_i \widehat{w}(f)}$. 
        Let $\delta\in (0,1)$ be a fixed constant and $\tau\in [0,1]$ be a threshold.
		We define our denoised fairness program to be
	{
		\begin{equation} \tag{\bf DFair}
		\label{eq:progdenoised}
		\begin{split}
		&\min_{f\in \calF} \textstyle{\frac{1}{N}\sum_{a\in [N]} L(f, \widehat{s}_a) ~ s.t.} 
		\textstyle{~(H^\top)^{-1} \widehat{u}(f) \geq (\lambda-M\delta) \mathbf{1}}, \\
		&\textstyle{ ~ \min_{i\in [p]} \Gamma_i(f)  \geq (\tau-\delta)\cdot \max_{i\in [p]} \Gamma_i(f).}
		\end{split}
		\end{equation}
	}
	\end{definition}
	{$\delta$ is used as a relaxation parameter depending on the context.
	By definition, we can regard $M$ as a metric that measures how noisy $H$ is.
	Intuitively, as diagonal elements $H_{ii}$ increases, eigenvalues of $H$ increase, and hence, $M$ decreases.
	Also note that $M\geq 1$ since $M$ is at least the largest eigenvalue of $H^{-1}$ and $H$ is a non-singular stochastic matrix whose largest eigenvalue is $1$.
	Intuitively, $\Gamma_i(f)$ is designed to estimate $\Pr\left[\xi(f)\mid \xi'(f), Z=i\right]$: its numerator approximates $\Pr\left[\xi(f),\xi'(f), Z=i\right]$ and its denominator approximates $\Pr\left[\xi'(f), Z=i\right]$.
	For the denominator,  since Definition~\ref{def:flippingnoise} implies $\Pr\left[\widehat{Z}=j \mid \xi'(f),  Z=i\right]\approx H_{ij}$,
	we can estimate $\Pr\left[\xi'(f), Z=i\right]$ by a linear combination of $\Pr\left[\xi'(f), \widehat{Z}=j\right]$, i.e., $(H^\top)^{-1}_i\widehat{w}(f)$. 
	Similar intuition is behind the estimate of the numerator $(H^\top)^{-1}_i\widehat{u}(f)$.
	}
	Due to how $\Gamma_i$s are chosen, the first constraint is designed to estimate Assumption~\eqref{assumption:ratio}, and the last constraint is designed to estimate $\Omega_q(f,S)\geq \tau$.

	This design ensures that an optimal fair classifier $f^\star$ satisfies our denoised constraints w.h.p.  (Lemma~\ref{lm:denoised1}), and hence, is a feasible solution to Program~\eqref{eq:progdenoised}.
	Consequently, the empirical risk of an optimal classifier $f^\Delta$ of Program~\eqref{eq:progtarget} is at most that of $f^\star$.
	The main difficulty is to prove that $f^\Delta$ achieves fairness on the underlying dataset $S$, since an unfair classifier may also satisfy our denoised constraints and is output as a feasible solution of Program~(\eqref{eq:progdenoised}).
	To handle this, we show all unfair classifiers that are infeasible for Program~\eqref{eq:progcon} should violate our denoised constraints (Lemma~\ref{lm:denoised2}).
	For this, we verify that the probability of each unfair classifier being feasible is exponentially small, and bound certain “capacity” of unfair classifiers (Definition~\ref{def:capacity}) using Assumption~\ref{assumption:ratio}.

	\subsection{Main theorem: Performance of  Program~\textbf{\eqref{eq:progdenoised}}}
	\label{sec:thm}

	Our main theorem shows that solving Program~\eqref{eq:progdenoised} leads to a classifier that does not increase the empirical risk (compared to the optimal fair classifier) and only slightly violates the fairness constraint.
	Before we state our result, we need the following definition that measures the complexity of $\calF$. 

	\begin{definition}[\bf{VC-dimension of $(S,\calF)$~\cite{har-peled2011geometric}}]
		\label{def:vc}
		Given a subset $A\subseteq [N]$, we define 
		$
		\textstyle \calF_A:=\left\{\left\{a\in A: f(s_a)=1 \right\} \mid f\in \calF \right\}
		$
		to be the collection of subsets of $A$ that may be shattered by some $f\in \calF$.
		The VC-dimension of $(S,\calF)$ is the largest integer $t$ such that there exists a subset $A\subseteq [N]$ with $|A|=t$ and $|\calF_A|=2^t$.
	\end{definition}
	
	\noindent
	Suppose $\calX\subseteq \R^d$ for some integer $d\geq 1$.
		If $\calF = \left\{0,1\right\}^\calX$, we observe that the VC-dimension is $t=N$.
		Several commonly used families $\calF$ have VC-dimension $O(d)$, including linear threshold functions~\cite{har-peled2011geometric}, kernel SVM and gap tolerant classifiers~\cite{burges1998a}.
		Using this definition, the main theorem in this paper is as follows.
	
		\begin{theorem}[\bf{Performance of Program~\eqref{eq:progdenoised}}]
		\label{thm:denoised}	
		Suppose the VC-dimension of $(S,\calF)$ is $t\geq 1$.
		Given any flipping noise matrix $H\in [0,1]^{p\times p}$, $\lambda\in (0,  0.5)$ and $\delta\in (0, 1)$, let $f^\Delta\in \calF$ denote an optimal fair classifier of Program~\eqref{eq:progdenoised}.
		With probability at least $1-O(p e^{-\frac{\lambda^2 \delta^2 n}{60000 M^2}+ t\ln(50M/\lambda \delta)})$, we have $\frac{1}{N} \sum_{a\in [N]} L(f^\Delta, s_a) \leq \frac{1}{N} \sum_{a\in [N]} L(f^\star, s_a)$ and $\Omega_q(f^\Delta,S)\geq \tau-3\delta$.
	\end{theorem}
	
	\noindent
	Theorem~\ref{thm:denoised} indicates that $f^\Delta$ is an approximate fair classifier for Problem~\ref{problem:general} with an exponentially small failure probability to the data size $n$.
	A few remarks are in order.

	\begin{remark}
		\label{remark:2}
		Observe that the success probability depends on $1/M$, $\delta$, $\lambda$ and the VC-dimension $t$ of $(S,\calF)$.
		If $1/M$ or $\delta$ is close to $0$, i.e., the protected attributes are very noisy or there is no relaxation for $\Omega_q(f,S)\geq \tau$ respectively, the success probability guarantee naturally tends to be $0$.
		Next, we discuss the remaining parameters $\lambda$ and $t$.

	\noindent	
		\textbf{Discussion on $\lambda$.}
		Intuitively, the success probability guarantee tends to $0$ when $\lambda$ is close to $0$.
		For instance, consider $q$ to be the statistical rate (Eq.~(\eqref{eq:gamma})). 
		Suppose there is only one sample $s_1$ with $Z=1$ for which $f^\star(s_1)=1$, i.e., $\Pr_D\left[f^\star=1, Z=1\right]=1/N$ and, therefore, $\lambda \leq 1/N$.
		To approximate $f^\star$, we may need to label $f(s_1)=1$.
		However, due to the flipping noises, it is likely that we can not find out the specific sample $s_1$ to label $f(s_1)=1$, unless we let the classifier prediction be $f = 1$ for all samples, which leads to a large empirical risk (see discussion in Section~\ref{sec:alg1}).
		In other words, the task is tougher for smaller values of $\lambda$.
		
	\noindent	
		\textbf{Discussion on $t$.}
		The success probability also depends on $t$ which captures the complexity of $\calF$.
		Suppose $\calX\subseteq \R^d$ for some integer $d\geq 1$.
		The worst case is $\calF = \left\{0,1\right\}^\calX$ with $t=N$, which takes the success probability guarantee to 0.
		On the other hand, if the VC-dimension does not depend on $N$, e.g., only depends on $d \ll N$, the failure probability is exponentially small on $N$.
		For instance, if $\calF$ is the collection of all linear threshold functions, i.e., each classifier $f\in \calF$ has the form $f(s_a) = \I\left[\langle x_a, \theta \rangle \geq r \right]$ for some vector $\theta\in \R^d$ and threshold $r\in \R$.
		We have $t\leq d+1$ for an arbitrary dataset $S$~\cite{har-peled2011geometric}.
	\end{remark}
	
	\begin{remark}
	\label{remark:ex_prob}
	   The $f^\Delta$ guaranteed by our theorem is \textbf{both} approximately fair and  optimal w.h.p.
	   This is in contrast to learning a stochastic classifier $\tilde{f}\sim \Lambda$ over $\calF$, that is in expectation near-optimal for both accuracy and fairness, e.g., $\mathbb{E}_{\tilde{f}\sim \Lambda}\left[\frac{1}{N} \sum_{a\in [N]} L(\tilde{f}, s_a) \right] \leq \frac{1}{N} \sum_{a\in [N]} L(f^\star, s_a)$ and $\mathbb{E}_{f\sim \Lambda}\left[\Omega_q(\tilde{f},S)\right]\geq \tau-3\delta$.
		For instance, suppose $f_1, f_2\in \calF$ such that the empirical risk of $f_1$ is $\frac{3}{2N} \sum_{a\in [N]} L(f^\star, s_a)$ and $\Omega(f_1,S)= \tau/2$, while the empirical risk of $f_2$ is $\frac{1}{2N} \sum_{a\in [N]} L(f^\star, s_a)$ and $\Omega(f_2,S)= 3\tau/2$.
If $\Lambda$  is uniform over $f_1$ and $f_2$,  it satisfies the above two  inequalities,
		But, neither of $f_i$s 
		is near-optimal for  accuracy and fairness.
	\end{remark}
	
	\noindent
	\paragraph{Estimation errors.}	In practice, we can use prior work on noise parameter estimation \cite{menon2015learning, liu2015classification, northcutt2017learning} to obtain estimates of $H$, say $H'$.
		The scale of estimation errors also affects the performance of our denoised program.
		In Appendix~\ref{sec:influence}, we provide a technical discussion on the effect of the estimation errors on the performance.
		Concretely, we consider a specific setting that $p=2$ and $q$ is the statistical rate.
		Define $\zeta := \max_{i,j\in [p]}|H_{ij}- H'_{ij}|$ to be the additive estimation error.
		We show there exists constant $\alpha > 0$ such that $\Omega_q(f^\Delta, S) \geq \tau - 3\delta - \zeta \alpha$ holds.
		Compared to Theorem~\ref{thm:denoised}, the estimation errors introduce an additive $\zeta \alpha$ error term for the fairness guarantee of our denoised program.

\subsection{Proof Overview of Theorem~\ref{thm:denoised} for $p=2$ and statistical rate}
	\label{sec:proof}
	
	For ease of understanding, we consider a specific case in the main body: a binary sensitive attribute $Z\in \left\{0,1\right\}$ and statistical rate constraints, i.e., 
	\begin{eqnarray}
	\label{eq:gamma}
	\begin{split}
	\textstyle \gamma(f, S):=\frac{\min_{i\in \left\{0,1\right\}}\Pr_{D}\left[f=1\mid Z=i\right]}{\max_{i\in\left\{0,1\right\}}\Pr_{D}\left[f=1\mid Z=i\right]} \geq \tau.
	\end{split}
	\end{eqnarray}
	Consequently, we would like to prove $\gamma(f^\Delta, S)\geq \tau-3\delta$ to obtain Theorem~\ref{thm:denoised}.
	The proof for the general Theorem~\ref{thm:denoised} can be found in Section~\ref{sec:complete_generalization}.
	We denote $\eta_0 = H_{01}$ to be the probability that $\widehat{Z}=1$ conditioned on $Z=0$, and $\eta_1 = H_{10}$ to be the probability that $\widehat{Z}=0$ conditioned on $Z=1$. 
	By Assumption~\ref{assumption:ratio}, we have $\eta_0, \eta_1 < 0.5$.
	Combining with Definition~\ref{def:flippingnoise}, we have $\footnotesize \textstyle{H = \begin{bmatrix}
	1-\eta_0 & \eta_0 \\
	\eta_1 & 1-\eta_1
	\end{bmatrix}}$,
	which implies that $M = \frac{1}{1-\eta_0-\eta_1}$.
	Consequently, Assumption~\ref{assumption:ratio} is equivalent to
	$
    \textstyle		\min_{i\in \left\{0,1\right\}} \Pr_{D}\left[ f^\star=1, Z=i \right] \geq \lambda,
	$
	and for $i\in \left\{0,1\right\}$,
	$
    \textstyle		\Gamma_i(f):= \frac{(1-\eta_{1-i})\Pr\left[f=1, \widehat{Z}=i\right]-\eta_1\Pr\left[f=1, \widehat{Z}=1-i\right]}{(1-\eta_{1-i})\widehat{\mu}_i-\eta_{1-i} \widehat{\mu}_{1-i}}.
    $
    We define the denoised statistical rate to be 
	$
	\gamma^\Delta(f,\widehat{S}) := \min\left\{\frac{\Gamma_0(f)}{\Gamma_1(f)}, \frac{\Gamma_1(f)}{\Gamma_0(f)}\right\},
	$
	and our denoised constraints become
	\begin{eqnarray}
	\label{eq:denoised}
	\scriptsize
	\begin{cases}
	&(1-\eta_{1-i})\Pr\left[f=1, \widehat{Z}=i\right]-\eta_{1-i}\Pr\left[f=1, \widehat{Z}={1-i}\right] \\
	&\geq  (1-\eta_0-\eta_1)\lambda-\delta,  \quad i\in \left\{0,1\right\}\\
	&\gamma^\Delta(f,\widehat{S}) \geq \tau-\delta,
	\end{cases}
	\end{eqnarray}

	\noindent
	\paragraph{Proof overview.}
	The proof of Theorem \ref{thm:denoised} relies on two lemmas: 1) The first shows that $f^\star$ is a feasible solution for  Constraints~(\eqref{eq:denoised}) (Lemma~\ref{lm:denoised1}). 
	The feasibility of $f^\star$ for the first constraint of  (\eqref{eq:denoised}) is guaranteed by Assumption~\ref{assumption:ratio} and for the second constraint of (\eqref{eq:denoised}) follows from the fact that $\Gamma_i(f)$ ($i\in \left\{0,1\right\}$) is a good estimation of $\Pr\left[\xi(f),\xi'(f)\mid Z=i\right]$ by the Chernoff bound.
	2) The second lemma shows that w.h.p. ($1-F$ for small $F$), all unfair classifiers $f\in \calF$ that are either not feasible for Program~\eqref{eq:progcon} or violate Assumption~\ref{assumption:ratio}, violate Constraint~(\eqref{eq:denoised}) (Lemma~\ref{lm:denoised2}).
	Since the space of unfair classifiers is continuous, the main difficulty is  to upper bound the (violating) probability $F$.
	Towards this, we first divide the collection of all unfair classifiers into multiple groups depending on how much they violate  Constraint~(\eqref{eq:denoised}) (Definition~\ref{def:bad}).
	Then, for each group $G_i$, we construct an $\eps_i$-net $\calG_i$ (Definition~\ref{def:net}) such that ensuring all classifiers $f\in \calG_i$ violate Constraint~(\eqref{eq:denoised}) is sufficient to guarantee that all classifiers in $\calG_i$ violate Constraint~(\eqref{eq:denoised}). 
	Here, $\eps_i$ is chosen to depend on the degree of violation of $G_i$.
	Using Chernoff bounds, we show that the probability each unfair classifier on the net $\calG_i$ is feasible to Constraint~(\eqref{eq:denoised}) is 
	$\exp(-O(1-\eta_0-\eta_1)^2\lambda^2 n))$.
	Hence, as $\lambda$ decreases, it is more likely that an unfair classifier is feasible for Constraint~(\eqref{eq:denoised}).
To bound the total violating probability,	it remains to bound the number of classifiers in the union of these nets (Definition~\ref{def:capacity}).
	The idea is to apply the relation between VC-dimension and $\eps$-nets (Theorem~\ref{thm:vc_net}).

	The two lemmas imply that the empirical risk of $f^\Delta$ is guaranteed to be at most that of $f^\star$ and $f^\Delta$ must be fair over $S$ (Theorem~\ref{thm:denoised}).
{Overall, the main technical contribution is to discretize the  space of unfair classifiers by carefully chosen multiple $\eps$-nets with different violation degrees to our denoised program, and upper bound the capacity of the union of all nets via a VC-dimension bound.}

	We now present the formal statements of the two main lemmas: Lemmas~\ref{lm:denoised1} and~\ref{lm:denoised2}, and defer all proofs to Section~\ref{sec:complete_proof}.

	\begin{lemma}[\bf{Relation between Program~\eqref{eq:progtarget} and~\eqref{eq:progdenoised}}]
		\label{lm:denoised1}
		Let $f\in \calF$ be an arbitrary classifier and $\eps\in (0,0.5)$.
		With probability at least $1-2e^{-\eps^2 n/6}$,
		\[
		\textstyle  (1-\eta_{1-i})\Pr\left[f=1,\widehat{Z}=i\right] - \eta_{1-i} \Pr\left[f=1, \widehat{Z}=1-i\right] 
		\in  (1-\eta_0-\eta_1)\Pr\left[f=1,Z=i\right] \pm \eps, 
		\]
		\sloppy
	    for $i\in \left\{0,1\right\}$.
		Moreover, if $\min_{i\in \left\{0,1\right\}}\Pr\left[f=1,Z=i\right]\geq \frac{\lambda}{2}$, then with probability at least $1-4e^{-\frac{\eps^2(1-\eta_0-\eta_1)^2\lambda^2 n}{2400}}$,
		$
		\textstyle \gamma^\Delta(f,\widehat{S}) \in (1\pm \eps)\gamma(f,S).
		$
	\end{lemma}
	
	\noindent
	The first part of this lemma shows how to estimate $\Pr\left[f=1, Z=i\right]$ ($i\in \left\{0,1\right\}$) in terms of $\Pr\left[f=1, \widehat{Z}=0\right]$ and $\Pr\left[f=1, \widehat{Z}=1\right]$, which motivates the first constraint of~(\eqref{eq:denoised}).
	The second part of the lemma motivates the second constraint of~(\eqref{eq:denoised}).
	Then by Assumption~\ref{assumption:ratio}, $f^\star$ is likely to be feasible for Program~\eqref{eq:progdenoised}.
	Consequently, $f^\Delta$ has empirical loss at most that of $f^\star$.
For our second main lemma,	we first define the collection of classifiers that are expected to violate. Constraint~(\eqref{eq:denoised}).

	\begin{definition}[\bf{Bad classifiers}]
		\label{def:bad}
		Given a family $\calF\subseteq \left\{0,1\right\}^\calX$, we call $f\in \calF$ a bad classifier if $f$ belongs to at least one of the following sub-families:
		\begin{itemize}
		{ \footnotesize	\item \sloppy $ \calG_0:= \big\{f\in \calF: \min\left\{\Pr\left[f=1,Z=0\right], \Pr\left[f=1,Z=1\right]\right\} < \frac{\lambda }{2} \big\}$;
			\item Let $  T=\lceil 232\log\log \frac{2(\tau-3\delta)}{\lambda} \rceil$. 
			For $i\in [T]$, define
			%	\[
			$ \calG_i:= \left\{f\in \calF\setminus \calG_0: \gamma(f,S) \in \left[\frac{\tau-3\delta }{1.01^{2^{i+1}-1}}, \frac{\tau-3\delta }{1.01^{2^i-1}} \right) \right\}.
			$	
			%\]
			%
			}
		\end{itemize}
	\end{definition}
	
	\noindent
	Intuitively, classifier $f\in \calG_0$ is likely to violate the first or the second of Constraint~(\eqref{eq:denoised});
	and for $f\in \calG_i$ for some $i\in [T]$ it is likely that $\gamma^\Delta(f, \widehat{S})< \tau-\delta$.
	Thus, any bad classifier is likely to violate Constraint~(\eqref{eq:denoised}) (Lemma~\ref{lm:denoised1}).
	Then we lower bound the total violating probability for all bad classifiers by the following lemma.

	\begin{lemma}[\bf{Bad classifiers are not feasible for Constraint~(\eqref{eq:denoised})}]
		\label{lm:denoised2}
		Suppose the VC-dimension of $(S,\calF)$ is $t$; then with probability at least $1-O\left(e^{-\frac{(1-\eta_0-\eta_1)^2 \lambda^2 \delta^2 n}{5000}+ t\ln(\frac{50}{(1-\eta_0-\eta_1)\lambda\delta})}\right)$, any bad classifier violates Constraint~(\eqref{eq:denoised}).
	\end{lemma}
	
	\noindent
	Theorem~\ref{thm:denoised} for $p=2$ and statistical rate is almost a direct corollary of Lemmas~\ref{lm:denoised1} and~\ref{lm:denoised2} (see Section \ref{sec:complete_proof}) except that we need to verify that any classifier violating Program (\eqref{eq:progdenoised}) is a bad classifier in the sense of Definition \ref{def:bad}.

\eat{	
	\begin{proofsketch}
		We discuss the cases of $\calG_0$ and $\calG_i$ ($i\in [T]$) separately.
		For $\calG_0$, let $G_0$ be an $\eps_0$-net of $\calG_0$ of size $M_{\eps_0}(\calG_0)$.
		By Lemma~\ref{lm:denoised1}, we can prove that with probability at least $1-2e^{-2\eps_0^2 n} M_{\eps_0}(\calG_0)$, all classifiers $g\in G_0$ violate either the first or the second constraint in (\eqref{eq:denoised}) by at least an additive $\frac{(1-\eta_0-\eta_1)\lambda}{2}$ term.
		Conditioned on this event, we can verify that all classifiers $f\in \calG_0$ violate at least one of the first two constraints of (\eqref{eq:denoised}) since there must exist a classifier $g \in G_0$ such that $\Pr\left[f\neq g\right]\leq \eps_0$.

		For $\calG_i$ ($i\in [T]$), the argument is similar: first construct an $\eps_i$-net $G_i$ of $\calG_i$, then show all classifiers in $G_i$ are not feasible with probability at least $1-4e^{-\frac{\eps_i^2(1-\eta_0-\eta_1)^2\lambda^2 n}{2400}}$ by Lemma~\ref{lm:denoised1}, and finally extend to all classifiers in $\calG_i$.
		The only difference is that each $f\in \calG_i$ violates the third constraint of (\eqref{eq:denoised}), say $\gamma^\Delta(f,\widehat{S})< \tau - \delta$ holds with high probability.
		The lemma is a direct corollary by the union bound.
	\end{proofsketch}
}

	\section{Empirical results}
	\label{sec:empirical}
	\setlength\fboxsep{0pt}
    \setlength{\tabcolsep}{4pt}

	\begin{table*}[t]
		\centering
		\caption{\small The performance on accuracy and fairness metrics of all algorithms over the test datasets; we report the average and standard error (in parenthesis). %
		When the protected attribute is binary, the fairness metrics (SR, FPR) are $\min_{i\in \set{0,1}} q_i(f){/}\max_{i\in \set{0,1}} q_i(f)$.
		For the non-binary protected attribute (COMPAS-race), we report the performance for all groups; i.e., SR$_j$, FPR$_j$ denote $ q_j(f)/ \max_{i\in [p]} q_i(f)$, for all $j{\in}[p]$.
		By definition, SR${=}\min\set{\text{SR}_j}$ and FPR${=}\min\set{\text{FPR}_j}$.
		The full accuracy-fairness tradeoffs when varying $\tau$ can be found in Appendix~\ref{sec:other}.
		For each dataset and protected attribute, the metrics of the method that achieves the largest sum of mean accuracy and mean statistical rate (one way to measure fairness-accuracy tradeoff) has also been colored in green, and the method that achieves the largest sum of mean accuracy and mean false positive rate  has been colored in yellow. Our method \textbf{DLR} achieves the best tradeoff or is within one standard deviation of the best tradeoff, as measured in this manner, in 6 out of 8 settings.
		}
		\label{tab:sr_results}
		\scriptsize
		
		\begin{tabular}{lp{0.6cm}p{0.6cm}p{0.6cm}|p{0.6cm}p{0.6cm}p{0.6cm}|p{0.6cm}p{0.6cm}p{0.6cm}|p{0.6cm}p{0.6cm}p{0.6cm}p{0.6cm}p{0.6cm}p{0.6cm}p{0.6cm}}
		
			\toprule
			\multirow{3}{*}{} & \multicolumn{6}{c|}{\textbf{Adult}} & \multicolumn{10}{c}{\textbf{COMPAS}} \\
			& \multicolumn{3}{c|}{sex (binary)} & \multicolumn{3}{c|}{race (binary)} &\multicolumn{3}{c|}{sex (binary)} &\multicolumn{7}{c}{race (non-binary)} \\
			& acc & \multicolumn{1}{c}{SR} & FPR  & acc  & \multicolumn{1}{c}{SR} & FPR & acc & SR & FPR  & acc & SR$_0$ & SR$_1$ & SR$_2$  & FPR$_0$ & FPR$_1$ & FPR$_2$   \\
			\midrule
			\scriptsize{\textbf{Unconstrained}} & \textbf{.80} (0) & .31 (.01) & .45 (.03) & \textbf{.80} (0) & .68 (.02) &   .81 (.09) & \textbf{.67} (.01) & .78 (.04) & .70 (.08)  & \textbf{.67} (0) & .66 (.02) & .96 (.01) & 1.0 (.0) & .57 (.02) & 1.0 (0) & .94 (.01)   \\
			\scriptsize{\textbf{LR-SR}} & .76 (.01) & .68 (.24)  & .68 (.21)  & .76 (.01) & .69 (.27) & .71 (.26) & .67 (.01) & .79 (.04) & .72 (.08) & .58 (.06) & .86 (.09) & .98 (.03) & .98 (.02) & .85 (.11) & .98 (.04) & .96 (.04) \\
			\scriptsize{\textbf{LR-FPR}} & .76 (.01) & .82 (.21) & .78 (.25) & .76 (0) & {.83} (.29)  & .84 (.29) & .67 (.02) & .80 (.04) & .72 (.08) & .56 (.05) & .87 (.08) & .97 (.06) & .97 (.03) & .86 (.09) & .96 (.07) & .95 (.09)   \\
			\scriptsize{\specialcell{\textbf{LZMV} $\varepsilon_L{=}.01$}} & .35 (.01)  & \textbf{.99} (0) & \textbf{.99} (0) & .37 (.05)  & \textbf{.98} (0) &\textbf{.99} (0)  & \colorbox{lime}{.5}\colorbox{yellow}{5} (.01)  & \colorbox{lime}{\textbf{.98}} (.04) & \colorbox{yellow}{\textbf{.98}} (.09)  & - & - &  - & - & - &  - & - \\
			\scriptsize{\specialcell{\textbf{LZMV} $\varepsilon_L{=}.04$}} &  .67 (.04) & .85 (.06) & \textbf{.99} (.01)  & .77 (.03)  & .79 (.10) & .85 (.09)  &  .58 (.01) & .94 (.02) & .94 (.03)  & - &  -  & - & - & - &  - & -\\
			\scriptsize{\specialcell{\textbf{LZMV} $\varepsilon_L{=}.10$}} & .78 (.02)  & .69 (.09)  & .79 (.11)   & .80 (0)  & .70 (.01)  & .82 (.08)    & .64 (.02)  & .85 (.05)  &  .81 (.07)  &  - & - & - & - & - &  - & - \\
			\scriptsize{\textbf{AKM}} &  .77 (0) & .66 (.05) & .89 (.04) &  \colorbox{yellow}{\textbf{.80}} (0) & .72 (.02)  & \colorbox{yellow}{.90} (.08) & .66 (.01) & .83 (.04) & .77 (.09) & -  & -  &  - & - & - &  - & - \\
			\scriptsize{\textbf{WGN+}} & .70 (.05) & .73 (.12) & .76 (.05)  & {.76 (.01)}  & .84 (.05) & {.92 (.05)} &  .59 (.01) & .90 (.02) & .84 (.01)    &  .56 (.02)  &  .89  (.14) &  .91 (.18) &  .96  (.13) & .85 (.16)  &  .87  (.23) &  .94  (.16) \\
			\midrule
			\scriptsize{\textbf{DLR-SR} $\tau{=}.7$} & .77 (.01) &  .74 (.14) & .87 (.17)   & .79 (.01)  & .80 (.12)  & .90 (.10)  & .67 (.01) & .79 (.04) & .72 (.08)  & .66 (.01) & .73 (.04) & .99 (.01) & 1.0 (0) & .66 (.05) & 1.0 (.0) & .92 (.03)  \\
			\scriptsize{\textbf{DLR-SR} $\tau{=}.9$} &  \colorbox{lime}{.76} (.01) & \colorbox{lime}{.85} (.15)  & .80 (.12) & \colorbox{lime}{.76} (.01)  & \colorbox{lime}{{.88}} (.18)  & .90 (.19)  & .63 (.04) & .86 (.05) & .83 (.08)  & \colorbox{lime}{.55} (.04) & \colorbox{lime}{.91} (.06) & .97 (.04) & .97 (.03) & .89 (.09) & .97 (.04) & .93 (.1) \\
			\scriptsize{\textbf{DLR-FPR} $\scriptsize \tau{=}.7$} &  .77 (.02) &  .73 (.14)  & .85 (.17)   & {.78 (.02)} &  .77 (.11) & {.88 (.11)}   & .66 (.01) & .80 (.04) & .73 (.08)   & .64 (.02) & .76 (.05) & .99 (.01) & .98 (.02) & .72 (.06) & 1.0 (.0) & .89 (.06)     \\
			\scriptsize{\textbf{DLR-FPR} $\tau{=}.9$} & \colorbox{yellow}{.77} (.02) &  .77 (.12)  & \colorbox{yellow}{.91} (.11)   & .77 (.02)& .80 (.15)   & .88 (.14)  & .60 (.06) & .86 (.07) & .82 (.10) & \colorbox{yellow}{.53} (.04) & \textbf{.92} (.06) & .97 (.06) & .95 (.06) & \colorbox{yellow}{\textbf{.93}} (.08) & .94 (.09) & .93 (.07) \\
			\bottomrule
		\end{tabular}	
	\end{table*}

	{

	We implement our denoised program, for binary and non-binary protected attributes, and compare the performance with baseline algorithms on real-world datasets. 

	\paragraph{Datasets.} 
	We perform simulations on the \textbf{Adult}~\cite{Asuncion+Newman:2007} and \textbf{COMPAS} \cite{compas} benchmark datasets, as pre-processed in AIF360 toolkit~\cite{aif360-oct-2018}. 
	The \textbf{Adult} dataset consists of rows corresponding to 48,842 individuals, with 18 binary features and a label indicating whether the income is greater than 50k USD or not. 
	We use binary protected attributes sex (``male'' ($Z{=}1$) vs ``female'' ($Z{=}0$)) and race (``White'' ($Z{=}1$) vs ``non-White'' ($Z{=}0$)) for this dataset.
	The \textbf{COMPAS} dataset consists of rows corresponding to 6172 individuals, with 10 binary features and a label that takes value 1 if the individual does not reoffend and 0 otherwise.
	We take sex (coded as binary) and race (coded as non-binary - ``African-American'' ($Z{=}1$), ``Caucasian'' ($Z{=}2$), ``Other'' ($Z{=}3$))
	to be the protected attributes.

	\paragraph{Metrics and baselines.} 
	We implement our program using logistic loss with denoised constraints with respect to the statistical rate and false positive rate metrics; we refer to our algorithm with statistical rate constraints as \textbf{DLR-SR} and with false positive rate constraints as \textbf{DLR-FPR}.
	\footnote{{We use the (noisy) protected attribute to construct the constraints, but not for classification. However, if necessary, the protected attribute can also be used as a feature for classification.}}
	{To obtain computationally feasible formulations of our optimization problem (\eqref{eq:denoised}), we expand the constraint on the fairness metrics by forming constraints on relevant (empirical) rates of all groups, and solve the nonconvex program using SLSQP; the details of the constraints are presented in Section~\ref{sec:other}.}
	We compare against state-of-the-art noise-tolerant fair classification algorithms: \textbf{LZMV}~\cite{lamy2019noise}, \textbf{AKM}~\cite{awasthi2020equalized}, and \textbf{WGN+}~\cite{wang2020robust}.
	\textbf{LZMV} takes as input a parameter, $\varepsilon_{L}$, to control the fairness of the final classifier; for statistical rate, this parameter represents the desired absolute difference between the likelihood of positive class label across the two protected groups and \textbf{LZMV} is, therefore, the primary baseline for comparison with respect to statistical rate.
	We present the results of \cite{lamy2019noise} for different $\varepsilon_{L}$ values.\footnote{{{github.com/AIasd/noise\textunderscore fairlearn}.}}
	\textbf{AKM}\footnote{{{github.com/matthklein/equalized\textunderscore odds\_under\_perturbation}.}} and \textbf{WGN+}\footnote{{{github.com/wenshuoguo/robust-fairness-code}.}} are the primary baseline for comparison with respect to false positive rate metric.
	As discussed earlier, the algorithm \textbf{AKM} is the post-processing algorithm of \citet{hardt2016equality}.
	\footnote{Equalized odds fairness metric aims for parity w.r.t false positive and true positive rates.
	For clarity of presentation, we present the empirical analysis with respect to false positive rate only. 
	}
	For \textbf{WGN+}, we use the algorithm that employs soft-group assignments \cite{kallus2020assessing} to form false positive rate constraints; it is the only prior algorithm that can handle non-binary noisy protected attributes and, hence, it is also the main baseline for the COMPAS dataset with race protected attribute.

	Additionally, we implement the baseline which minimizes the logistic loss with fairness constraints ($\tau={0.9}$) over the noisy protected attribute as described in Section~\ref{sec:alg2}.
	When the fairness metric is the statistical rate, we will refer to this program as \textbf{LR-SR}, and when the fairness metric is the false positive rate, we will refer to it as \textbf{LR-FPR}.
	Finally, we also learn an unconstrained optimal classifier as a baseline.

	\paragraph{Implementation details.}
	We first shuffle and partition the dataset into a train and test partition (70-30 split).
	Given the training dataset $S$, we generate a noisy dataset $\widehat{S}$. For binary protected attributes, we use $\eta_0 = 0.3$ and $\eta_1 = 0.1$. For non-binary protected attributes, we use the noise matrix {\scriptsize $H = \begin{bmatrix} 0.70 & 0.15 & 0.15 \\ 0.05 & 0.90 & 0.05 \\ 0.05 & 0.05 & 0.90 \end{bmatrix}$} (i.e., the minority group is more likely to contain errors, as would be expected in various applications \cite{melissa2000shades}).
	{Our algorithms, as well as the baselines, have access to the known $\eta$ and $H$ values.}
	We consider other choices of noise parameters and impact of error in estimates of noise parameter in Appendix~\ref{sec:other}.
	We train each algorithm on $\widehat{S}$ and vary the fairness constraints (e.g., the choice of $\tau\in [0.5, 0.95]$ in \textbf{DLR}), learn the corresponding fair classifier, and record its accuracy (\textrm{acc}) and fairness metric (either statistical rate or false positive rate) $\gamma$ over the noisy version of the test dataset. 
	We perform 50 repetitions and report the mean and standard error of fairness and accuracy metrics across the repetitions.
	For the COMPAS, we use $\lambda = 0.1$ as a large fraction (47\%) of training samples have class label 1, while for Adult, we use $\lambda = 0$ as the fraction of positive class labels is small (24\%).
	\footnote{Alternately, one could use a range of values for $\lambda$ to construct multiple classifiers, and choose the one which satisfies the program constraints and has the best accuracy over a separate validation partition. We find that these $\lambda$ are sufficient to obtain fair classifiers for the considered datasets.}

	\paragraph{Results.}
	Table~\ref{tab:sr_results} summarizes the fairness and accuracy achieved by our methods and baseline algorithms over the {Adult} and {COMPAS} test datasets.
	The first observation is that our approach, \textbf{DLR-SR} and \textbf{DLR-FPR}, achieve higher fairness than the unconstrained classifier, showing its effectiveness in noise-tolerant fair classification. The extent of this improvement varies with the strength of the constraint $\tau$, but comes with a natural tradeoff with accuracy. 

	For {Adult} dataset, \textbf{DLR-SR} and \textbf{DLR-FPR} (with $\tau{=}0.9$) can attain a higher fairness metric value than \textbf{LR-SR} and \textbf{LR-FPR} respectively, and perform similarly with respect to accuracy.
    The statistical rate-accuracy tradeoff of \textbf{DLR-SR}, for this dataset, is also better than \textbf{LZMV}, \textbf{AKM}, and \textbf{WGN+}; in particular, high statistical rate for Adult dataset using \textbf{LZMV} (i.e., ${\geq}0.8$) is achieved only with a relatively larger loss in accuracy (for example, with $\epsilon_L{=}0.01$), whereas for \textbf{DLR-SR}, the loss in accuracy when using $\tau{=}0.9$ is relatively small (${\sim}0.03$) while the statistical rate is still high (${\sim}0.85$).	
	With respect to false positive rate, \textbf{AKM} can achieve a high false positive rate for the Adult dataset (${\sim}0.90$), while \textbf{WGN+} does not achieve high false positive rate when sex is the protected attribute.
	In comparison, \textbf{DLR-FPR} with $\tau{=}0.9$ can also achieve a high false positive rate at a small loss of accuracy for both protected attributes, and the best false positive rate and accuracy of \textbf{DLR-FPR} and \textbf{AKM} are within a standard deviation of each other. 
	Baseline \textbf{LZMV} attains a high false positive rate too for the Adult dataset, but the loss in accuracy is larger compared to \textbf{DLR-FPR}.

	For the COMPAS dataset, with sex as protected attribute, \textbf{LZMV} ($\varepsilon_L{=}0.01,0.04$) achieves high statistical rate and false positive rate, but at a large cost to accuracy. Meanwhile \textbf{DLR-SR} ($\tau{=}0.9$) returns a classifier with SR $\sim 0.86$ and FPR $\sim 0.83$ and significantly better accuracy (0.63) than \textbf{LZMV}($\varepsilon_L{=}0.01,0.04$).
	Further, our algorithm can achieve higher fairness as well, at the cost of accuracy, using a larger input $\tau$ (e.g., $\tau{=}1$; see Appendix~\ref{sec:other}).
	Note that in this case, the unconstrained classifier already has high fairness values.
    Hence, despite the noise in protected attribute, the task of fair classification is relatively easy and all baselines, as well as, our methods perform well for this dataset and protected attribute.

	For the COMPAS dataset with non-binary race protected attribute, we also present the complete breakdown of relative performance for each protected attribute value in Table~\ref{tab:sr_results}.
	Both \textbf{DLR-SR} and \textbf{DLR-FPR} (with $\tau=0.9$) reduce the disparity between group-performances $q_j(f)$ and $\max_{i \in [p]}$ $\forall j \in [p]$, for SR and FPR metrics, to a larger extent compared to the unconstrained classifier, baselines and \textbf{WGN+}.

	The tradeoff between the fairness metric and accuracy for all methods is also graphically presented in Appendix~\ref{sec:other}.
	Evaluation with respect to both metrics shows that our framework can handle binary and non-binary protected attributes, and	attain close to the user-desired fairness metric values (as defined using $\tau$).
	Comparison with baselines further shows that, unlike \textbf{AKM} and \textbf{WGN+}, our approach can always return classifiers with high fairness metrics values, and unlike \textbf{LZMV}, the loss in accuracy to achieve high fairness values is relatively small.
	We also present the performance of our approach using false discovery rate (linear-fractional metric) constraints in Section~\ref{sec:other}; in that setting, our approach has better fairness-accuracy tradeoff than baselines for \textbf{Adult} and similar tradeoff as the best-performing baseline for \textbf{COMPAS}.
 }

\section{Missing proofs in Section~\ref{sec:proof}}
\label{sec:complete_proof}

In this section, we complete the missing proofs in Section~\ref{sec:proof}.
Let $ \textstyle{\pi_{ij}:=\Pr_{D,\widehat{D}}\left[\widehat{Z}=i\mid Z=j\right]}$ for $i,j\in \left\{0,1\right\}$, $\mu_i:= \Pr_{D}\left[Z=i\right]$ and $\textstyle{\widehat{\mu}_i:= \Pr_{\widehat{D}}\left[\widehat{Z}=i\right]}$ for $i\in \left\{0,1\right\}$.

\subsection{Proof of Lemma~\ref{lm:denoised1}}
\label{sec:proof_lm_denoised1}

\begin{proof}
	We first have the following simple observation.

	\begin{observation}
		\label{observation:prob}
		1) $\mu_0+\mu_1 = 1$, $\widehat{\mu}_0 + \widehat{\mu}_1 = 1$, and $\pi_{0,i} + \pi_{1,i} = 1$ holds for $i\in \left\{0,1\right\}$;
		2) For any $i,j\in \left\{0,1\right\}$,
		$
		\Pr\left[Z=i\mid \widehat{Z}=j\right] = \frac{\pi_{ji}\cdot \mu_i}{\widehat{\mu}_j};
		$
		3) For any $i\in \left\{0,1\right\}$,
		$
		\widehat{\mu}_i = \pi_{i,i}\mu_i + \pi_{i,1-i}\mu_{1-i}.
		$
	\end{observation}
	
	\noindent
	Similar to Equation~\eqref{eq:gap2_1}, we have
	\begin{eqnarray}
	\small
	\label{eq:lm1}
	\begin{split}
    \Pr\left[f=1, \widehat{Z}=0\right] 
	& = && \Pr\left[\widehat{Z}=0\mid f=1,Z=0\right] \cdot \Pr\left[f=1, Z=0\right] \\
	& && +\Pr\left[\widehat{Z}=0\mid f=1,Z=1\right] \cdot \Pr\left[f=1, Z=1\right].
	\end{split}
	\end{eqnarray}
	Similar to the proof of Lemma~\ref{lm:gap2}, by the Chernoff bound (additive form) \cite{hoeffding1994probability}, both
	\begin{align}
	\small
	\label{ineq:lm1}
	\Pr\left[\widehat{Z}=1\mid f=1,Z=0\right] \in \eta_0 \pm \frac{\eps}{2\Pr\left[f=1,Z=0\right]}, 
	\end{align}
	and
	\begin{align}
	\small
	\label{ineq:lm2}
	\Pr\left[\widehat{Z}=0\mid f=1,Z=1\right] \in \eta_1 \pm \frac{\eps}{2\Pr\left[f=1,Z=1\right]},
	\end{align}
	hold with probability at least 
	\begin{align*}
	1-2e^{-\frac{\eps^2 n }{12\eta\Pr\left[f=1,Z=0\right]}}-2e^{-\frac{\eps^2 n }{12\eta\Pr\left[f=1,Z=1\right]}}
	\end{align*}
	which for $\eta\leq  0.5$, is at least $1-2e^{-\eps^2 n/6}$.
	Consequently, we have
	\begin{eqnarray}
	\small
	\label{eq:lm_denoised1_1}
	\begin{split}
	& && \Pr\left[f=1, \widehat{Z}=0\right] \\
	& = && \Pr\left[\widehat{Z}=0\mid f=1,Z=0\right] \cdot \Pr\left[f=1, Z=0\right] +\Pr\left[\widehat{Z}=0\mid f=1,Z=1\right] \cdot \Pr\left[f=1, Z=1\right] \\ 
	& && (\text{Eq.~\eqref{eq:lm1}}) \\
	& \in && \left(1-\eta_0 \pm \frac{\eps}{2\Pr\left[f=1,Z=0\right]}\right)\cdot \Pr\left[f=1, Z=0\right]\\ & + &&\left(\eta_1 \pm \frac{\eps}{2\Pr\left[f=1,Z=1\right]}\right)\cdot \Pr\left[f=1, Z=1\right] (\text{Ineqs.~\eqref{ineq:lm1} and~\eqref{ineq:lm2}}) \\
	&\in && (1-\eta_0)\Pr\left[f=1,Z=0\right]  + \eta_1\Pr\left[f=1,Z=1\right]\pm \eps, 
	\end{split}
	\end{eqnarray}
	and similarly,
	\begin{eqnarray}
	\small
	\label{eq:lm_denoised1_2}
	\begin{split}
	\Pr\left[f=1,\widehat{Z}=1\right] \in \eta_0\Pr\left[f=1,Z=0\right] + (1-\eta_1)\Pr\left[f=1,Z=1\right]\pm \eps.
	\end{split}
	\end{eqnarray}
	By the above two inequalities, we conclude that 
	\begin{align*}
	& &&\small (1-\eta_1)\Pr\left[f=1, \widehat{Z}=0\right] - \eta_1 \Pr\left[f=1, \widehat{Z}=1\right] \\
	& \in && (1-\eta_1) \big((1-\eta_0)\Pr\left[f=1, Z=0\right] + \eta_1 \Pr\left[f=1, Z=1\right] \pm \eps\big) - \eta_1 \big(\eta_0\Pr\left[f=1, Z=0\right] \\
	& &&+ (1-\eta_1) \Pr\left[f=1, Z=1\right] \pm \eps\big) \quad (\text{Ineqs.~\eqref{eq:lm_denoised1_1} and~\eqref{eq:lm_denoised1_2}}) \\
	&\in && (1-\eta_0 - \eta_1) \Pr\left[f=1, Z=0\right] \pm \eps. 
	\end{align*}
	Similarly, we have
	\begin{eqnarray*}
	\small
		(1-\eta_0)\Pr\left[f=1, \widehat{Z}=1\right] - \eta_0 \Pr\left[f=1, \widehat{Z}=0\right]
		\in (1-\eta_0-\eta_1) \Pr\left[f=1, Z=1\right] \pm \eps.
	\end{eqnarray*}
	This completes the proof of the first conclusion.

	Next, we focus on the second conclusion. 
	By assumption, $\min\left\{\Pr\left[f=1,Z=0\right], \Pr\left[f=1,Z=1\right] \right\}\geq \frac{\lambda}{2}$.
	Let $\eps' = \frac{\eps (1-\eta_0-\eta_1) \lambda}{20}$.
	By a similar argument as for the first conclusion, we have the following claim. 
	
	\begin{claim}
		\label{claim:lm}
		With probability at least $1- 4e^{-(\eps')^2 n/6}$, we have 
		\begin{eqnarray*}
	\small
			\begin{cases}
				&(1-\eta_1)\Pr\left[f=1,\widehat{Z}=0\right] - \eta_1 \Pr\left[f=1, \widehat{Z}=1\right] \\
				\in &(1-\eta_0-\eta_1)\Pr\left[f=1,Z=0\right] \pm \eps', \\
				&(1-\eta_0)\Pr\left[f=1,\widehat{Z}=1\right] - \eta_0 \Pr\left[f=1, \widehat{Z}=0\right] \\
				\in & (1-\eta_0-\eta_1)\Pr\left[f=1,Z=1\right] \pm \eps', \\
				& (1-\eta_1)\widehat{\mu}_0 - \eta_1 \widehat{\mu}_1 \in (1-\eta_0-\eta_1) \mu_0 \pm \eps', \\
				& (1-\eta_0)\widehat{\mu}_1 - \eta_0 \widehat{\mu}_0 \in (1-\eta_0-\eta_1) \mu_1 \pm \eps'.
			\end{cases}
		\end{eqnarray*}
	\end{claim}

	\noindent
	Now we assume Claim~\ref{claim:lm} holds whose success probability is at least $1-4e^{-\frac{\eps^2(1-\eta_0-\eta_1)^2\lambda^2 n}{2400}}$ since $\eps' = \frac{\eps (1-\eta_0-\eta_1) \lambda}{20}$.
	Consequently, we have
	\begin{eqnarray}
	\small
	\label{ineq:lm3}
	\begin{split}
    (1-\eta_1)\Pr\left[f=1,\widehat{Z}=0\right] - \eta_1 \Pr\left[f=1, \widehat{Z}=1\right]
	& \geq &&(1-\eta_0-\eta_1)\Pr\left[f=1,Z=0\right] -\eps' \quad (\text{Claim~\ref{claim:lm}})\\
	& \geq && \frac{(1-\eta_0-\eta_1)\lambda}{2} -\eps' \quad (\text{by assumption}) \\
	& \geq &&  0.45 \cdot (1-\eta_0-\eta_1)\lambda. \\
	& && (\eps'=\frac{\eps (1-\eta_0-\eta_1)\lambda}{20})
	\end{split} 
	\end{eqnarray}
	Similarly, we can also argue that
	\begin{align}
	\label{ineq:lm4}
	(1-\eta_1)\widehat{\mu}_0 - \eta_1 \widehat{\mu}_1 \geq  0.45 \cdot (1-\eta_0-\eta_1)\lambda.
	\end{align}
	Then we have
	\begin{eqnarray*}
	\small
		\label{ineq:lm5}
		\begin{split}
			\Pr\left[f=1\mid Z=0\right]
			& = && \frac{\Pr\left[f=1, Z=0\right]}{\mu_0} \\
			&\in && \frac{(1-\eta_1)\Pr\left[f=1,\widehat{Z}=0\right] - \eta_1 \Pr\left[f=1, \widehat{Z}=1\right] \pm \eps'}{(1-\eta_1)\widehat{\mu}_0 - \eta_1 \widehat{\mu}_1 \pm \eps'} \\
			& && (\text{Claim~\ref{claim:lm}})\\
			& \in && \frac{\left((1-\eta_1)\Pr\left[f=1,\widehat{Z}=0\right] - \eta_1 \Pr\left[f=1, \widehat{Z}=1\right]\right)}{(1\pm \frac{\eps'}{ 0.45 \cdot (1-\eta_0-\eta_1)\lambda})\left((1-\eta_1)\widehat{\mu}_0 - \eta_1 \widehat{\mu}_1\right)} \\
			& && \times (1\pm \frac{\eps'}{ 0.45\cdot (1-\eta_0-\eta_1)\lambda})\quad (\text{Ineq.~\eqref{ineq:lm3}}) \\
			& \in && (1\pm \frac{\eps}{9})^2 \cdot\Gamma_0(f). \quad(\text{Defns. of $\Gamma_0(f)$ and $\eps'$})
		\end{split}
	\end{eqnarray*}
	Similarly, we can also prove that
	\begin{eqnarray*}
		\label{ineq:lm6}
		\Pr\left[f=1\mid Z=1\right] \in (1\pm \frac{\eps}{9})^2 \cdot\Gamma_1(f).
	\end{eqnarray*}
	By the above two inequalities, we have that with probability at least $1-4e^{-\frac{\eps^2(1-\eta_0-\eta_1)^2\lambda^2 n}{2400}}$,
	\begin{align*}
	\small
    \gamma^\Delta(f,S)
	& = && \min\left\{\frac{\Gamma_0(f)}{\Gamma_1(f)}, \frac{\Gamma_1(f)}{\Gamma_0(f)}\right\} \\
	& \in && (1\pm \eps) \times \min\left\{\frac{\Pr\left[f=1\mid Z=0\right]}{\Pr\left[f=1\mid Z=1\right]},\frac{\Pr\left[f=1\mid Z=1\right]}{\Pr\left[f=1\mid Z=0\right]}\right\} \\
	& \in && (1\pm \eps)\cdot \gamma(f,S). 
	\end{align*}
	Combining with Claim~\ref{claim:lm}, we complete the proof of the second conclusion.
\end{proof}

\subsection{Proof of Lemma~\ref{lm:denoised2}}
\label{sec:proof_lm_denoised2}

For preparation, we give the following definition.

	\begin{definition}[\bf{$\eps$-nets}]
		\label{def:net}
		Given a family $\calF\subseteq \left\{0,1\right\}^\calX$ of classifiers and $\eps\in (0,1)$, we say $F\subseteq \calF$ is an $\eps$-net of $\calF$ if for any $f,f'\in F$, $\Pr_D\left[f\neq f'\right]\geq \eps$; and for any $f\in \calF$, there exists $f'\in F$ such that $\Pr_D\left[f\neq f'\right]\leq \eps$. 
		We denote $M_\eps(\calF)$ as the smallest size of an $\eps$-net of $\calF$.
	\end{definition}
	
	\noindent
	It follows from basic coding theory~\cite{lint1998introduction} that $M_\eps(\left\{0,1\right\}^\calX) = \Omega(2^{N-O(\eps N \log N)})$.
	The size of an $\eps$-net usually depends exponentially on the VC-dimension.

	\begin{theorem}[\bf{Relation between VC-dimension and $\eps$-nets~\cite{haussler1995sphere}}]
		\label{thm:vc_net}
		Suppose the VC-dimension of $(S,\calF)$ is $t$.
		For any $\eps\in (0,1)$, $M_\eps(\calF) = O(\eps^{-t})$.
	\end{theorem}
	
	\noindent
	We define the capacity of bad classifiers based on $\eps$-nets.

	\begin{definition}[\bf{Capacity of bad classifiers}]
		\label{def:capacity}
		Let $\eps_0 = \frac{(1-\eta_0-\eta_1)\lambda-2\delta}{5}$.
		Let $\eps_i = \frac{1.01^{2^{i-1}} \delta}{5}$ for $i\in [T]$ where $T=\lceil 232\log\log \frac{2(\tau-3\delta)}{\lambda} \rceil$.
		Given $\calF\subseteq \left\{0,1\right\}^\calX$, we denote the capacity of bad classifiers by
		$
		\textstyle \Phi(\calF) := 2e^{-\eps_0^2 n/6} M_{\eps_0}(\calG_0)+ 4\sum_{i\in [T]} e^{-\frac{\eps_i^2(1-\eta_0-\eta_1)^2\lambda^2 n}{2400}}  M_{\eps_i (1-\eta_0-\eta_1)\lambda/10}(\calG_i).
		$
	\end{definition}
	
	\noindent
	Actually, we can prove $\Phi(\calF)$ is an upper bound for the probability that there exists a bad classifier that is feasible for Program~\eqref{eq:progdenoised}, which is a generalized version of Lemma~\ref{lm:denoised2}.
	Roughly, the factor $2e^{-\eps_0^2 n/6}$ is an upper bound of the probability that a bad classifier $f\in \calG_0$ violates Constraint~(\eqref{eq:denoised}), and the factor $4e^{-\eps_i^2\lambda^2 \delta^2 n}$ is an upper bound of the probability that a bad classifier $f\in \calG_i$ violates Constraint~(\eqref{eq:denoised}).
	We prove if all bad classifiers in the nets of $\calG_i$ ($0\leq i\leq T$) are not feasible for Program~\eqref{eq:progdenoised}, then all bad classifiers should violate Constraint~(\eqref{eq:denoised}). 
	Note that the scale of $\Phi(\calF)$ depends on the size of $\eps$-nets of $\calF$, which can be upper bounded by Theorem~\ref{thm:vc_net} and leads to the success probability of Theorem~\ref{thm:denoised}.

\begin{proof}
    We first claim that Lemma~\ref{lm:denoised2} holds with probability at least $1-\Phi(\calF)$.
	We discuss $\calG_0$ and $\calG_i$ ($i\in [T]$) separately.

	\paragraph{Bad classifiers in $\calG_0$.}
	Let $G_0$ be an $\eps_0$-net of $\calG_0$ of size $M_{\eps_0}(\calG_0)$.
	Consider an arbitrary classifier $g\in G_0$.
	By Lemma~\ref{lm:denoised1}, with probability at least $1-2e^{-\eps_0^2 n/6}$, we have
	\begin{eqnarray}
	\small
	\label{eq:lm_denoised2_1}
	\begin{split}
	 (1-\eta_1)\Pr\left[g=1,\widehat{Z}=0\right] -\eta_1 \Pr\left[g=1, \widehat{Z}=1\right] 
	& \leq&& (1-\eta_0-\eta_1)\Pr\left[g=1,Z=0\right] + \eps_0  \\
	& < && 	\frac{(1-\eta_0-\eta_1)\lambda}{2}+ \eps_0, \quad (\text{Defn. of $\calG_0$})
	\end{split}
	\end{eqnarray}
	and
	\begin{eqnarray}
	\small
	\label{eq:lm_denoised2_1'}
	\begin{split}
    (1-\eta_0)\Pr\left[g=1,\widehat{Z}=1\right] -\eta_0 \Pr\left[g=1, \widehat{Z}=0\right] 
	& < && 	\frac{(1-\eta_0-\eta_1)\lambda}{2}+ \eps_0.
	\end{split}
	\end{eqnarray}
	By the union bound, all classifiers $g\in G_0$ satisfy Inequalities~\eqref{eq:lm_denoised2_1} and~\eqref{eq:lm_denoised2_1'} with probability at least $1-2e^{-\eps_0^2 n/6} M_{\eps_0}(\calG_0)$.
	Suppose this event happens.
	We consider an arbitry classifier $f\in \calG_0$. 
	W.l.o.g., we assume $\Pr\left[f=1,Z=0\right]<\frac{\lambda}{2}$.
	By Definition~\ref{def:net}, there must exist a classifier $g \in G_0$ such that $\Pr\left[f\neq g\right]\leq \eps_0$.
	Then we have
	\begin{align*}
	\small
	& && (1-\eta_1)\Pr\left[f=1,\widehat{Z}=0\right] -  \eta_1 \Pr\left[f=1, \widehat{Z}=1\right] & \\
	& \leq && (1-\eta_1)(\Pr\left[g=1,\widehat{Z}=0\right] +\eps_0)  -\eta_1 (\Pr\left[g=1, \widehat{Z}=1\right]-\eps_0) \quad (\Pr\left[f\neq g\right]\leq \eps_0) \\
	& \leq && \frac{(1-\eta_0-\eta_1)\lambda}{2}+ 2\eps_0 \quad (\text{Ineq.~\eqref{eq:lm_denoised2_1}}) \\
	& \leq && \frac{(1-\eta_0-\eta_1)\lambda}{2} + \frac{(1-\eta_0-\eta_1)\lambda-2\delta}{2} (\text{Defn. of $\eps_0$}) \\
	& = && (1-\eta_0-\eta_1)\lambda - \delta,
	\end{align*}
	Thus, we conclude that all classifiers $f\in \calG_0$ violate Constraint~\eqref{eq:denoised} with probability at least $1-2e^{-\eps_0^2 n/6} M_{\eps_0}(\calG_0)$.

	\paragraph{Bad classifiers in $\calG_i$ for $i\in [T]$.}
	We can assume that $\tau - 3\delta \geq \lambda/2$.
	Otherwise, all $\calG_i$ for $i\in [T]$ are empty, and hence, we complete the proof.
	Consider an arbitry $i\in [T]$ and let $G_i$ be an $\eps_i$-net of $\calG_i$ of size $M_{\eps_i (1-\eta_0-\eta_1)\lambda/10}(\calG_i)$.
	Consider an arbitrary classifier $g\in G_i$.
	By the proof of Lemma~\ref{lm:denoised1}, with probability at least $1-4e^{-\frac{\eps_i^2(1-\eta_0-\eta_1)^2\lambda^2 n}{2400}}$, we have
	\begin{eqnarray}
	\small
	\label{eq:lm_denoised2_6}
	\begin{cases}
	&(1-\eta_1)\Pr\left[g=1,\widehat{Z}=0\right] - \eta_1 \Pr\left[g=1, \widehat{Z}=1\right]  \in (1-\eta_0-\eta_1)\Pr\left[g=1,Z=0\right] \pm \frac{\eps_i(1-\eta_0-\eta_1)\lambda}{20}, \\
	&(1-\eta_0)\Pr\left[g=1,\widehat{Z}=1\right] - \eta_0 \Pr\left[g=1, \widehat{Z}=0\right]  \in (1-\eta_0-\eta_1)\Pr\left[g=1,Z=1\right] \pm \frac{\eps_i(1-\eta_0-\eta_1)\lambda}{20}, \\
	& \gamma^\Delta(f,\widehat{S}) \in (1\pm \eps_i)\cdot \gamma(f,S).
	\end{cases}
	\end{eqnarray}
	Moreover, we have
	\begin{eqnarray}
	\small
	\label{eq:lm_denoised2_2}
	\begin{split}
	\gamma^\Delta(g, \widehat{S}) &\leq && (1+\eps_i)\cdot \gamma(g,S) < (1+\eps_i)\cdot \frac{\tau-3\delta }{1.01^{2^i-1}}. (\text{Defn. of $\calG_i$})
	\end{split}
	\end{eqnarray}
	By the union bound, all classifiers $g\in G_i$ satisfy Inequality~\eqref{eq:lm_denoised2_2} with probability at least 
	\[
	1-4e^{-\frac{\eps_i^2(1-\eta_0-\eta_1)^2\lambda^2 n}{2400}} M_{\eps_i (1-\eta_0-\eta_1)\lambda/10}(\calG_i).
	\]
	Suppose this event happens.
	We consider an arbitry classifier $f\in \calG_i$. 
	By Definition~\ref{def:net}, there must exist a classifier $g \in G_i$ such that $\Pr\left[f\neq g\right]\leq \eps_i (1-\eta_0-\eta_1)\lambda/10$.
	By Inequality~\eqref{eq:lm_denoised2_6} and a similar argument as that for Inequality~\eqref{ineq:lm3}, we have
	\begin{eqnarray}
	\small
	\label{eq:lm_denoised2_7}
	\begin{split}
	(1-\eta_1)\Pr\left[g=1, \widehat{Z}=0\right] - \eta_1 \Pr\left[g=1, \widehat{Z}=1\right]
	\geq 0.45 \cdot (1-\eta_0-\eta_1)\lambda.
	\end{split}
	\end{eqnarray}
	\begin{eqnarray}
	\small
	\label{eq:lm_denoised2_4}
	\begin{split}
	\Gamma_0(f) &
	= \frac{(1-\eta_1)\Pr\left[f=1, \widehat{Z}=0\right] - \eta_1 \Pr\left[f=1, \widehat{Z}=1\right]}{(1-\eta)\widehat{\mu}_0-\eta \widehat{\mu}_1} \\
	\in & \frac{(1-\eta_1)\left(\Pr\left[g=1, \widehat{Z}=0\right] \pm \frac{\eps_i (1-\eta_0-\eta_1)\lambda }{10}\right) }{(1-\eta)\widehat{\mu}_0-\eta \widehat{\mu}_1}  - \frac{\eta_1 \left(\Pr\left[g=1, \widehat{Z}=1\right]  \pm \frac{\eps_i (1-\eta_0-\eta_1)\lambda}{10} \right)}{(1-\eta)\widehat{\mu}_0-\eta \widehat{\mu}_1} \\
	& (\Pr\left[f\neq g\right]\leq \eps_i (1-\eta_0-\eta_1)\lambda/10) \\
	\in & \frac{(1-\eta_1)\Pr\left[g=1, \widehat{Z}=0\right] - \eta_1 \Pr\left[g=1, \widehat{Z}=1\right]  }{(1-\eta)\widehat{\mu}_0-\eta \widehat{\mu}_1}  \pm  \frac{\frac{\eps_i (1-\eta_1-\eta_1)\lambda}{5}}{(1-\eta)\widehat{\mu}_0-\eta \widehat{\mu}_1} \\
	\in & \frac{(1-\eta)\Pr\left[g=1, \widehat{Z}=0\right] - \eta \Pr\left[g=1, \widehat{Z}=1\right] }{(1-\eta)\widehat{\mu}_0-\eta \widehat{\mu}_1}\times   (1\pm  0.45\eps_i)\quad (\text{Ineq.~\eqref{eq:lm_denoised2_7}}) \\
	\in & (1\pm  0.45\eps_i)\cdot\Gamma_0(g). 
	\end{split}
	\end{eqnarray}
	Similarly, we can also prove
	\begin{eqnarray}
	\label{eq:lm_denoised2_5}
	\Gamma_1(f)\in (1\pm  0.45\eps_i)\cdot\Gamma_1(g).
	\end{eqnarray}
	Thus, we conclude that
	\begin{eqnarray*}
	\small
		\begin{split}
			\gamma^\Delta(f,\widehat{S}) & = && \min\left\{\frac{\Gamma_0(f)}{\Gamma_1(f)}, \frac{\Gamma_1(f)}{\Gamma_0(f)} \right\} && \\
			& \leq && \frac{1+ 0.45\eps_i}{1- 0.45\eps_i} \cdot \min\left\{\frac{\Gamma_0(g)}{\Gamma_1(g)}, \frac{\Gamma_1(g)}{\Gamma_0(g)} \right\} && (\text{Ineqs.~\eqref{eq:lm_denoised2_4} and~\eqref{eq:lm_denoised2_5}}) \\
			& < &&\frac{1+ 0.45\eps_i}{1- 0.45\eps_i} \cdot (1+\eps_i)\cdot \frac{\tau-3\delta }{1.01^{2^i-1}}&& (\text{Ineq.~\eqref{eq:lm_denoised2_2}}) \\
			& \leq && \frac{1+ 0.45\eps_1}{1- 0.45\eps_1} \cdot (1+\eps_1)\cdot (\tau-3\delta) && (\text{Defn. of $\eps_i$})\\
			&\leq && \tau-\delta. \quad (\eps_1 = \frac{1.01 \delta}{5})
		\end{split}
	\end{eqnarray*}
	It implies that all classifiers $f\in \calG_i$ violate Constraint~\eqref{eq:denoised} with probability at least 
	\[
	1-4e^{-\frac{\eps_i^2(1-\eta_0-\eta_1)^2\lambda^2 n}{2400}} M_{\eps_i (1-\eta_0-\eta_1)\lambda/10}(\calG_i).
	\]
	By the union bound, we complete the proof of Lemma~\ref{lm:denoised2} for $\delta\in (0, 0.1\lambda)$.

	For general $\delta\in (0,1)$, each bad classifier violates Constraint~\eqref{eq:denoised} with probability at most $4e^{-\frac{\eps_1^2(1-\eta_0-\eta_1)^2\lambda^2 n}{2400}}$ by the above argument.
	By Definition~\ref{def:bad},
	$
	|M_{\eps_0}(\calG_0)| + \sum_{i\in [T]} |M_{\eps_i (1-\eta_0-\eta_1)\lambda/10}(\calG_i)|\leq |M_{\eps_1(1-\eta_0-\eta_1)\lambda/10}(\calF)|.
	$
	Then by the definition of $\Phi(\calF)$ and Theorem~\ref{thm:vc_net}, the probability that there exists a bad classifier violating Constraint~\eqref{eq:denoised} is at most
	$
	\Phi(\calF) = O\left(e^{-\frac{(1-\eta_0-\eta_1)^2 \lambda^2 \delta^2 n}{60000}+ t\ln(\frac{50}{(1-\eta_0-\eta_1)\lambda\delta})}\right).
	$
	This completes the proof of Lemma~\ref{lm:denoised2}.
\end{proof}

\subsection{Proof of Theorem~\ref{thm:denoised} for $p=2$ and statistical rate}
\label{sec:proof_denoised}

	\begin{proof}
	\sloppy
		We first upper bound the probability that $\gamma^\Delta(f^\Delta, \widehat{S})\geq \tau-3\delta$.
		Let $\calF_b = \left\{f\in \calF: \gamma(f,S)< \tau-3\delta\right\}$.
		If all classifiers in $\calF_b$ violate Constraint~(\eqref{eq:denoised}), we have that $\gamma^\Delta(f^\Delta, \widehat{S})\geq \tau-3\delta$.
		Note that if
		$
		\min_{i\in \left\{0,1\right\}}\Pr\left[f=1,Z=i\right]\geq \frac{\lambda}{2},
		$ 
		then $\gamma(f,S) \geq \frac{\lambda}{2}$ holds by definition.
		Also, $\frac{\lambda-3\delta}{1.01^{2^{T+1}-1}}\leq \frac{\lambda}{2}$.
		Thus, we conclude that
		$
		\calF_b \subseteq \cup_{i=0}^{T} \calG_i. 
		$
		Then if all bad classifiers violate Constraint~(\eqref{eq:denoised}), we have $\gamma^\Delta(f^\Delta, \widehat{S})\geq \tau-3\delta$.
		By Lemma~\ref{lm:denoised2}, $\gamma^\Delta(f^\Delta, \widehat{S})\geq \tau-3\delta$ holds with probability at least $1-O\left(e^{-\frac{(1-\eta_0-\eta_1)^2 \lambda^2 \delta^2 n}{60000}+ t\ln(\frac{50}{(1-\eta_0-\eta_1)\lambda\delta})}\right)$.

		Next, we upper bound the probability that $f^\star$ is feasible for Program~\eqref{eq:progdenoised}, which implies $\frac{1}{N} \sum_{a\in [N]} L(f^\Delta, s_a) \leq \frac{1}{N} \sum_{a\in [N]} L(f^\star, s_a)$.
		Letting $\eps = \delta$ in Lemma~\ref{lm:denoised1}, we have that with probability at least $1-2e^{-\delta^2 n/6}-4e^{-\frac{(1-\eta_0-\eta_1)^2 \lambda^2\delta^2 n}{2400}}$,
		\begin{eqnarray*}
			\begin{cases}
				&(1-\eta)\Pr\left[f^\star=1,\widehat{Z}=0\right] - \eta \Pr\left[f^\star=1, \widehat{Z}=1\right] \geq (1-\eta_0-\eta_1)\Pr\left[f^\star=1,Z=0\right] - \delta, \\
				&(1-\eta )\Pr\left[f^\star=1,\widehat{Z}=1\right] - \eta \Pr\left[f^\star=1, \widehat{Z}=0\right] \geq (1-\eta_0-\eta_1)\Pr\left[f^\star=1,Z=1\right] - \delta, \\
				&\gamma^\Delta(f^\star,\widehat{S}) \geq (1-\delta)\gamma(f,S)\geq \gamma(f,S) -\delta.
			\end{cases}
		\end{eqnarray*}
		It implies that $f^\star$ is feasible for Program~\eqref{eq:progdenoised} with probability at least $1-2e^{-\delta^2 n/6}-4e^{-\frac{(1-\eta_0-\eta_1)^2 \lambda^2\delta^2 n}{2400}}$.
		This completes the proof. 
	\end{proof}

\section{Proof of Theorem~\ref{thm:denoised} for multiple protected attributes and general fairness constraints}
	\label{sec:complete_generalization}

	In this section, we prove Theorem~\ref{thm:denoised} and show how to extend the theorem to multiple protected attributes and multiple fairness constraints (Remark~\ref{remark:multiple}).
	Denote $\calQ_{\mathrm{linf} }$ to be the collection of all group performance functions.
	Denote 
	$\textstyle{\calQ_{\mathrm{lin}}\subseteq \calQ_{\mathrm{linf} }}$ to be the collection of linear group performance functions.

	\color{black}
	\begin{remark} \label{rem:metric_comparison}
	The fairness metric considered in \cite{awasthi2020equalized}, i.e., equalized odds, can also be captured by $\textstyle{\calQ_{\mathrm{linf} }}$; equalized odds simply requires equal false positive and true positive rates across the protected types.
	The fairness metrics used in \cite{lamy2019noise}, on the other hand, are somewhat different; they work with statistical parity and equalized odds for binary protected attributes, however, 
	while we define disparity $\Omega_q$ as the ratio between the minimum and maximum $q_i$, \cite{lamy2019noise} define the disparity using the additive difference of $q_i$ across the protected types.
	{
	It is not apparent how to extend their method for improving additive metrics to linear-fractional fairness metrics as they counter the noise by scaling the tolerance of their constraints, and it is unclear how to compute these scaling parameters prior to the optimization step when the group performance function $q$ is conditioned on the classifier prediction.
	On the other hand, our method can handle additive metrics by using the difference of altered $q_i$ across the noisy protected attribute to form fairness constraints.
	}
	\end{remark}
	\color{black}

	Similar to Eq~(\eqref{eq:lm1}), we first have for each $i\in [p]$
	\begin{eqnarray*}
	\begin{split} 
	\textstyle \Pr\left[\xi'(f), \widehat{Z}=i\right] =\sum_{j\in [p]} \Pr\left[\widehat{Z}=i\mid \xi'(f), Z=j\right] \Pr\left[\xi'(f), Z=j\right].
	\end{split}
	\end{eqnarray*}
	By Definition~\ref{def:flippingnoise} and a similar argument as in the proof of Lemma~\ref{lm:denoised1}, we have the following lemma.

	\begin{lemma}[\bf{Relation between $\Pr\left[\xi'(f), \widehat{Z}=i\right]$ and $\Pr\left[\xi'(f), Z=j\right]$}]
		\label{lm:general1}
		Let $\eps\in (0,1)$ be a fixed constant.
		With probability at least $1-2p e^{-\eps^2 n/6}$, we have for each $i\in [p]$,
		\[
		\textstyle \Pr\left[\xi'(f), \widehat{Z}=i\right] \in \sum_{j\in [p]} H_{ji}\cdot \Pr\left[\xi'(f), Z=j\right] \pm \eps.
		\]
	\end{lemma}
	
	\noindent
	Define 
	\[
	\textstyle w(f) := \left(\Pr\left[\xi'(f),Z=1\right], \ldots, \Pr\left[\xi'(f),Z=p\right]\right),
	\] 
	and recall that
	\[
	\textstyle \widehat{w}(f) := \left(\Pr\left[\xi'(f), \widehat{Z}=1\right], \ldots, \Pr\left[\xi'(f), \widehat{Z}=p\right]\right).
	\]
	By Lemma~\ref{lm:general1}, we directly obtain the following lemma.
	
	\begin{lemma}[\bf{Approximation of $\Pr\left[\xi'(f), Z=i\right]$}]
		\label{lm:general2}
		With probability at least $1-2p e^{-\eps^2 n/6}$, for each $i\in [p]$,
		\[
		\text{\footnotesize $ w(f)_i \in (H^\top)^{-1}_i \widehat{w}(f) \pm \eps \|(H^\top)^{-1}_i\|_1 \in (H^\top)^{-1}_i \widehat{w}(f) \pm \eps M.$}
		\]
	\end{lemma}
	
	\noindent
	Thus, we use $(H^\top)^{-1}_i \widehat{w}(f)$ to estimate $\Pr\left[\xi'(f),Z=i\right]$.
	Similarly, we define 
    \begin{align*} 
    u(f) := \left(\Pr\left[\xi(f), \xi'(f),Z=i\right]\right)_{i\in [p]},
    \end{align*}
    and recall that
	\begin{align*}\widehat{u}(f) := \left(\Pr\left[\xi(f), \xi'(f), \widehat{Z}=i\right]\right)_{i\in [p]}.
	\end{align*}
	Once again, we use $(H^\top)^{-1}_i \widehat{u}(f)$ to estimate $\Pr\left[\xi(f), \xi'(f),Z=i\right]$ and to estimate constraint 
	$$
	\textstyle \min_{i\in [p]}\Pr\left[\xi(f), \xi(f), \xi'(f), Z=i \right]\geq \lambda,
	$$ 
	we construct the following constraint:
	\begin{align}
	\label{eq:gencon2}
	\textstyle (H^\top)^{-1} \widehat{u}(f) \geq (\lambda-\eps M) \mathbf{1},
	\end{align}
	which is the first constraint of Program~(\eqref{eq:progdenoised}).

	\eat{
	\begin{remark}
		\label{remark:4}
		The first two constraints of Program~\eqref{eq:progdenoised} are a special case of Constraint~(\eqref{eq:gencon1}). 
		By Definition~\ref{def:flippingnoise}, we have that
		\[
	    H = \textstyle{\begin{bmatrix}
		1-\eta_0 & \eta_0 \\
		\eta_1 & 1-\eta_1
		\end{bmatrix}} \text{ and }
		\]
		\[		\textstyle (H^\top)^{-1} = \begin{bmatrix}
		\frac{1-\eta_1}{1-\eta_0-\eta_1} & -\frac{\eta_1}{1-\eta_0-\eta_1} \\
		-\frac{\eta_0}{1-\eta_0-\eta_1} & \frac{1-\eta_0}{1-\eta_0-\eta_1}
		\end{bmatrix}.\]
		Then $M = \frac{1}{1-\eta_0-\eta_1}$ and we can verify that the first two constraints of Program~\eqref{eq:progdenoised} are equivalent to Constraint~(\eqref{eq:gencon1}) when $\xi'(f) = (f=1)$.
	\end{remark}
	}

	\noindent
	To provide the performance guarantees on the solution of the above program, once again we define the following general notions of bad classifiers and the corresponding capacity.

	\begin{definition}[\bf{Bad classifiers in general}]
		\label{def:bad_general}
		Given a family $\calF\subseteq \left\{0,1\right\}^\calX$, we call $f\in \calF$ a bad classifier if $f$ belongs to at least one of the following sub-families:
		\begin{itemize}
			\item \small $\calG_0:= \left\{f\in \calF: \min_{i\in [p]}\Pr\left[\xi(f), \xi'(f),Z=i\right] < \frac{\lambda }{2} \right\}$;
			\item Let $\textstyle{T=\lceil 232\log\log \frac{2(\tau-3\delta)}{\lambda} \rceil}$. 
			For $i\in [T]$, define
			\[
			\small
			\textstyle \calG_i:= \left\{f\in \calF\setminus \calG_0: \Omega_q(f,S) \in [\frac{\tau-3\delta }{1.01^{2^{i+1}-1}}, \frac{\tau-3\delta }{1.01^{2^i-1}} ) \right\}.
			\]
		\end{itemize}
	\end{definition}
	
	\noindent
	Note that Definition~\ref{def:bad} is a special case of the above definition by letting $p=2$, $M=10$, $\xi(f) = (f=1)$ and $\xi'(f)=\emptyset$.
	We next propose the following definition of the capacity of bad classifiers.

	\begin{definition}[\bf{Capacity of bad classifiers in general}]
		\label{def:capacity_general}
		Let $\eps_0 = \frac{\lambda-2\delta}{5M}$.
		Let $\eps_i = \frac{1.01^{2^{i-1}} \delta}{5}$ for $i\in [T]$ where $T=\lceil 232\log\log \frac{2(\tau-3\delta)}{\lambda} \rceil$.
		Given a family $\calF\subseteq \left\{0,1\right\}^\calX$, we denote the capacity of bad classifiers by
		\begin{align*} 
	\Phi(\calF):=  2p e^{-\eps_0^2 n/6} M_{\eps_0}(\calG_0) + 4p\sum_{i\in [T]} e^{-\frac{\eps_i^2\lambda^2 n}{2400M^2}} \cdot M_{\eps_i\lambda/10M}(\calG_i).
		\end{align*}
	\end{definition}
	
	\noindent
	By a similar argument as in Lemma~\ref{lm:denoised2}, we can prove that $\Phi(\calF)$ is an upper bound of the probability that there exists a bad classifier feasible for Program~\eqref{eq:progdenoised}.
	Now we are ready to prove Theorem~\ref{thm:denoised}.
	Actually, we prove the following generalized version.

	\begin{theorem}[\bf{Performance of Program~\eqref{eq:progdenoised}}]
		\label{thm:denoised2}	
		Suppose the VC-dimension of $(S,\calF)$ is $t\geq 1$.
		Given any non-singular matrix $H\in [0,1]^{p\times p}$ with $\sum_{j\in [p]} H_{ij} = 1$ for each $i\in [p]$ and $\lambda\in (0,  0.5)$, let $f^\Delta\in \calF$ denote an optimal fair classifier of Program~\eqref{eq:progdenoised}.
		With probability at least $1-\Phi(\calF)-4p e^{-\frac{\lambda^2 \delta^2 n}{2400 M^2}}$, the following properties hold
		\begin{itemize}
			\item $\frac{1}{N} \sum_{a\in [N]} L(f^\Delta, s_a) \leq \frac{1}{N} \sum_{a\in [N]} L(f^\star, s_a)$;
			\item $\Omega_q(f^\Delta,S)\geq \tau-3\delta$.
		\end{itemize}
		Specifically, if the VC-dimension of $(S,\calF)$ is $t$ and $\delta\in (0,1)$, the success probability is at least $1-O(p e^{-\frac{\lambda^2 \delta^2 n}{60000 M^2}+ t\ln(50M/\lambda \delta)})$.
	\end{theorem}
	
% 	\begin{proof}
% 		The proof is almost the same as in Theorem~\ref{thm:denoised}: we just need to replace $\frac{1}{1-\eta_0-\eta_1}$ by $M$ everywhere.
% 		%
% 		Note that the term $4p e^{-\frac{\lambda^2 \delta^2 n}{2400 M^2}}$ is an upper bound of the probability that $f^\star$ is not feasible for Program~\eqref{eq:progdenoised}.
% 		%
% 		The idea comes from Lemma~\ref{lm:general2} by letting $\eps = \frac{\lambda \delta}{20 M}$ such that for each $i\in [p]$,
% 		\[
% 		w(f^\star)_i \in (1\pm \frac{\delta}{10}) (H^\top)^{-1}_i \widehat{w}(f^\star) \text{ and } 
% 		\]
% 		\[
% 		u(f^\star)_i \in (1\pm \frac{\delta}{10}) (H^\top)^{-1}_i \widehat{u}(f^\star).
% 		\]
% 		%
% 		Consequently,  $\frac{1}{N} \sum_{a\in [N]} L(f^\Delta, s_a) \leq \frac{1}{N} \sum_{a\in [N]} L(f^\star, s_a)$.
% 		%
% 		Since $\Phi(\calF)$ is an upper bound of the probability that there exists a bad classifier feasible for Program~\eqref{eq:progdenoised}, we complete the proof.
% 	\end{proof}
	
	\noindent
	The proof is almost the same as in Theorem~\ref{thm:denoised}: we just need to replace $\frac{1}{1-\eta_0-\eta_1}$ by $M$ everywhere.
	For multiple fairness constraints, the success probability of Theorem~\ref{thm:denoised} changes to be 
	\[
	\textstyle 1-O(kp e^{-\frac{\lambda^2 \delta^2 n}{60000 M^2}+ t\ln(50M/\lambda \delta)}).
	\]

\begin{proof}
	Note that the term $4p e^{-\frac{\lambda^2 \delta^2 n}{2400 M^2}}$ is an upper bound of the probability that $f^\star$ is not feasible for Program~\eqref{eq:progdenoised}.
	The idea comes from Lemma~\ref{lm:general2} by letting $\eps = \frac{\lambda \delta}{20 M}$ such that for each $i\in [p]$,
	\[
	w(f^\star)_i \in (1\pm \frac{\delta}{10}) (H^\top)^{-1}_i \widehat{w}(f^\star) \text{ and } 
	\]
	\[
	u(f^\star)_i \in (1\pm \frac{\delta}{10}) (H^\top)^{-1}_i \widehat{u}(f^\star).
	\]
	Consequently,  $\frac{1}{N} \sum_{a\in [N]} L(f^\Delta, s_a) \leq \frac{1}{N} \sum_{a\in [N]} L(f^\star, s_a)$.
	Since $\Phi(\calF)$ is an upper bound of the probability that there exists a bad classifier feasible for Program~\eqref{eq:progdenoised}, we complete the proof.
\end{proof}

\begin{remark}[\bf{Generalization to multiple protected attributes and multiple fairness metrics}]
    \label{remark:multiple}
    For the general case that $m,k\geq 1$, i.e., there exists $m$ protected attributes $Z_1\in [p_1],\ldots, Z_m\in [p_m]$ and $k$ group performance functions $q^{(1)}, \ldots, q^{(l)}$ together with a threshold vector $\tau\in [0,1]^k$ where each $q^{(l)}$ is on some protected attribute. 
    In this case, we need to make a generalized assumption of Assumption~\ref{assumption:ratio}, i.e., there exists constant $\lambda\in (0, 0.5)$ such that for any $l\in [k]$,
    \[
    \textstyle{\min_{i\in [p]}\Pr_D\left[\xi^{(l)}(f^\star), (\xi')^{(l)}(f^\star), Z=i \right]\geq \lambda}.
    \]
    The arguments are almost the same except that for each group performance function $q^{(i)}$, we need to construct corresponding denoised constraints and have an individual capacity $\phi^{(i)}(\calF)$.
    Consequently, the success probability of Theorem~\ref{thm:denoised2} becomes $1-O\left(\sum_{i\in [m]} \phi^{(i)}(\calF)\right)$.
\end{remark}

	\section{Conclusion, limitations \& future work}
	\label{sec:conclusion}
	
	In this paper, we study fair classification with noisy protected attributes.
    We consider flipping noises and propose a unified framework that constructs an approximate optimal fair classifier over the underlying dataset for multiple, non-binary protected attributes and multiple linear-fractional fairness constraints.
	Our framework outputs a classifier that is guaranteed to be both fair and accurate.
	Empirically, our denoised algorithm can achieve the high fairness values at a small cost to accuracy.
    Thus this work broadens the class of settings where fair classification techniques can be applied by working even when the information about protected attributes is noisy.

    Our framework can be applied to a wide class of fairness metrics, and hence may be suitable in many domains. 
    However, it is not apriori clear which fairness metrics should be used in any given setting, and the answers will be very context-dependent; the effectiveness of our framework towards mitigating bias will depend crucially on whether the appropriate choice of features and parameters are selected.
    An ideal implementation of our framework would involve an active dialogue between the users and designers, a careful assessment of impact both pre and post-deployment. 
    This would in particular benefit from regular public audits of fairness constraints,  as well as ways to obtain and incorporate community feedback from stakeholders \cite{sassaman2020creating,chancellor2019relationships}.

	Our work leaves several interesting future directions.
	{One is to consider other noise models for non-binary attributes that are not independent, e.g., settings where the noise follows a general mutually contaminated model \cite{scott2013classification} or when the noise on the protected type also depends on other features, such as, when imputing the protected attributes. Our framework can still be employed in these settings (e.g., given group prediction error rates); however, methods that take into account the protected attribute prediction model could potentially further improve the performance.}
	There exist several works that also design fair classifiers with noisy labels~\cite{blum2020recovering,Biswas2020EnsuringFU} and another direction is to consider joint noises over both protected attributes and labels.
	Our model is also related to the setting in which each protected attribute follows a known distribution; whether our methods can be adapted to this setting can be investigated as part of future work.

    \section*{Acknowledgements}
    
    This research was supported in part by a J.P. Morgan Faculty Award and an AWS MLRA grant.

\bibliography{references}
\bibliographystyle{plainnat}
	
\clearpage
\appendix

\section{Analysis of the influences of estimation errors}
\label{sec:influence}
        We discuss the influences of estimation errors by considering a simple setting as in Section~\ref{sec:proof}, say $p=2$ with statistical rate.
		Recall that we assume $\eta_0$ and $\eta_1$ are given in Theorem~\ref{thm:denoised}.
		However, we may only have estimations for $\eta_0$ and $\eta_1$ in practice, say $\eta'_0$ and $\eta'_1$ respectively.
		Define $\zeta := \max\left\{|\eta_0 - \eta'_0|, |\eta_1 - \eta'_1|\right\}$ to be the additive estimation error.
		We want to understand the influences of $\zeta$ on the performance of our denoised program.

		Since $\eta_0$ and $\eta_1$ are unknown now, we can not directly compute $\Gamma_0(f)$ and $\Gamma_1(f)$ in Definition~\ref{def:denoised}.
		Instead, we can compute
		\begin{align*}
	\small
		\Gamma'_0(f) :=
		\frac{(1-\eta'_1)\Pr\left[f=1, \widehat{Z}=0\right]-\eta'_1\Pr\left[f=1, \widehat{Z}=1\right]}{(1-\eta'_1)\widehat{\mu}_0-\eta'_1 \widehat{\mu}_1},
		\end{align*}
		\begin{align*}
	\small
		\Gamma_1(f):=
		\frac{(1-\eta'_0)\Pr\left[f=1, \widehat{Z}=1\right]-\eta'_0\Pr\left[f=1, \widehat{Z}=0\right]}{(1-\eta'_0)\widehat{\mu}_1-\eta'_0\widehat{\mu}_0}.
		\end{align*}
		Then we have
		\begin{eqnarray}
	\small
		\label{eq:remark1}
		\begin{split}
		\Gamma'_0(f)
		& = && \frac{(1-\eta'_1)\Pr\left[f=1, \widehat{Z}=0\right]-\eta'_1\Pr\left[f=1, \widehat{Z}=1\right]}{(1-\eta'_1)\widehat{\mu}_0-\eta'_1 \widehat{\mu}_1} \\
		& = && \frac{(1-\eta_1)\Pr\left[f=1, \widehat{Z}=0\right]-\eta_1\Pr\left[f=1, \widehat{Z}=1\right] }{(1-\eta_1)\widehat{\mu}_0-\eta_1 \widehat{\mu}_1 + (\eta_1 - \eta'_1) }  + \frac{(\eta_1 - \eta'_1) \Pr\left[f=1\right]}{(1-\eta_1)\widehat{\mu}_0-\eta_1 \widehat{\mu}_1 + (\eta_1 - \eta'_1)} \\
		&\in && \frac{(1-\eta_1)\Pr\left[f=1, \widehat{Z}=0\right]-\eta_1\Pr\left[f=1, \widehat{Z}=1\right]}{(1-\eta_1)\widehat{\mu}_0-\eta_1 \widehat{\mu}_1}  \pm \frac{\zeta \cdot \Pr\left[f=1\right]}{(1-\eta_1)\widehat{\mu}_0-\eta_1 \widehat{\mu}_1} \quad  (\text{Defn. of $\zeta$}) \\
		& \in && \Gamma_0(f) \pm \frac{\zeta \cdot \Pr\left[f=1\right]}{(1-\eta_1)\widehat{\mu}_0-\eta_1 \widehat{\mu}_1}. \quad  (\text{Defn. of $\Gamma_0(f)$}) 
		\end{split}
		\end{eqnarray}
		Symmetrically, we have
		\begin{eqnarray}
		\label{eq:remark2}
		\Gamma'_1(f) \in \Gamma_1(f) \pm \frac{\zeta \cdot \Pr\left[f=1\right]}{(1-\eta_0)\widehat{\mu}_1-\eta_0 \widehat{\mu}_0}.
		\end{eqnarray}
		By a similar argument, we can also prove that
		\begin{eqnarray}
	\small
		\label{eq:remark3}
		\begin{split}
		\frac{1}{\Gamma'_0(f)} \in \frac{1}{\Gamma_0(f)}
		\pm \frac{\zeta }{(1-\eta_1)\Pr\left[f=1, \widehat{Z}=0\right]-\eta_1 \Pr\left[f=1, \widehat{Z}=1\right]}.
		\end{split}
		\end{eqnarray}
		and
		\begin{eqnarray}
	\small
		\label{eq:remark4}
		\begin{split}
		\frac{1}{\Gamma'_1(f)} \in \frac{1}{\Gamma_1(f)}
		\pm \frac{\zeta }{(1-\eta_0)\Pr\left[f=1, \widehat{Z}=1\right]-\eta_0 \Pr\left[f=1, \widehat{Z}=0\right]}.
		\end{split}
		\end{eqnarray}
		Then by the denoised constraint on $\eta'_0$ and $\eta'_1$, i.e.,
		\begin{align}
		\label{eq:remark5}
		\min\left\{\frac{\Gamma'_1(f)}{\Gamma'_0(f)}, \frac{\Gamma'_0(f)}{\Gamma'_1(f)}\right\} \geq \tau - \delta,
		\end{align}
		we conclude that
		\begin{align*}
	    \small
		\frac{\Gamma_1(f)}{\Gamma_0(f)} 
		& && \geq  \left(\Gamma'_1(f) - \frac{\zeta \cdot \Pr\left[f=1\right]}{(1-\eta_0)\widehat{\mu}_1-\eta_0 \widehat{\mu}_0}\right)\times \big(\frac{1}{\Gamma'_0(f)}  -  
		 \frac{\zeta }{(1-\eta_1)\Pr\left[f=1, \widehat{Z}=0\right]-\eta_1 \Pr\left[f=1, \widehat{Z}=1\right]}\big) \\
		& && (\text{Ineqs.~\eqref{eq:remark2} and~\eqref{eq:remark3}}) \\
		& && \geq \frac{\Gamma'_1(f)}{\Gamma'_0(f)} - \zeta\big(\frac{\Pr\left[f=1\right]}{\Gamma'_0(f)\left((1-\eta_0)\widehat{\mu}_1-\eta_0 \widehat{\mu}_0\right)} +\frac{\Gamma'_1(f)}{(1-\eta_1)\Pr\left[f=1, \widehat{Z}=0\right]-\eta_1 \Pr\left[f=1, \widehat{Z}=1\right]}\big) \\
		& && \geq \tau - \delta - \zeta \alpha_1, \quad (\text{Ineqs.~\eqref{eq:remark5}})
		\end{align*}
		where $\alpha_1 = \frac{\Pr\left[f=1\right]}{\Gamma'_0(f)\left((1-\eta_0)\widehat{\mu}_1-\eta_0 \widehat{\mu}_0\right)}+ \frac{\Gamma'_1(f)}{(1-\eta_1)\Pr\left[f=1, \widehat{Z}=0\right]-\eta_1 \Pr\left[f=1, \widehat{Z}=1\right]}$.
		Similarly, by Inequalities~\eqref{eq:remark1} and~\eqref{eq:remark4}, we have
		\[
		\frac{\Gamma_0(f)}{\Gamma_1(f)} \geq \tau - \delta - \zeta \alpha_0,
		\]
		where $\alpha_0 = \frac{\Pr\left[f=1\right]}{\Gamma'_1(f)\left((1-\eta_1)\widehat{\mu}_0-\eta_1 \widehat{\mu}_1\right)}+ \frac{\Gamma'_0(f)}{(1-\eta_0)\Pr\left[f=1, \widehat{Z}=1\right]-\eta_0 \Pr\left[f=1, \widehat{Z}=0\right]}$.
		Thus, we have 
		\[
		\gamma^\Delta(f, \widehat{S}) \geq \tau - \delta - \zeta\cdot \max\left\{\alpha_0, \alpha_1\right\}.
		\]
		The influence of the above inequality is that the fairness guarantee of Theorem~\ref{thm:denoised} changes to be
		\[
		\gamma(f^\Delta,S) \geq \tau - 3(\delta + \zeta\cdot \max\left\{\alpha_0, \alpha_1\right\}),
		\]
		i.e., the estimation errors will weaken the fairness guarantee of our denoised program.
		Also, observe that the influence becomes smaller as $\zeta$ goes to  0.

	\color{black}

\color{black}
	
\begin{figure*}
\centering
\includegraphics[width=\linewidth]{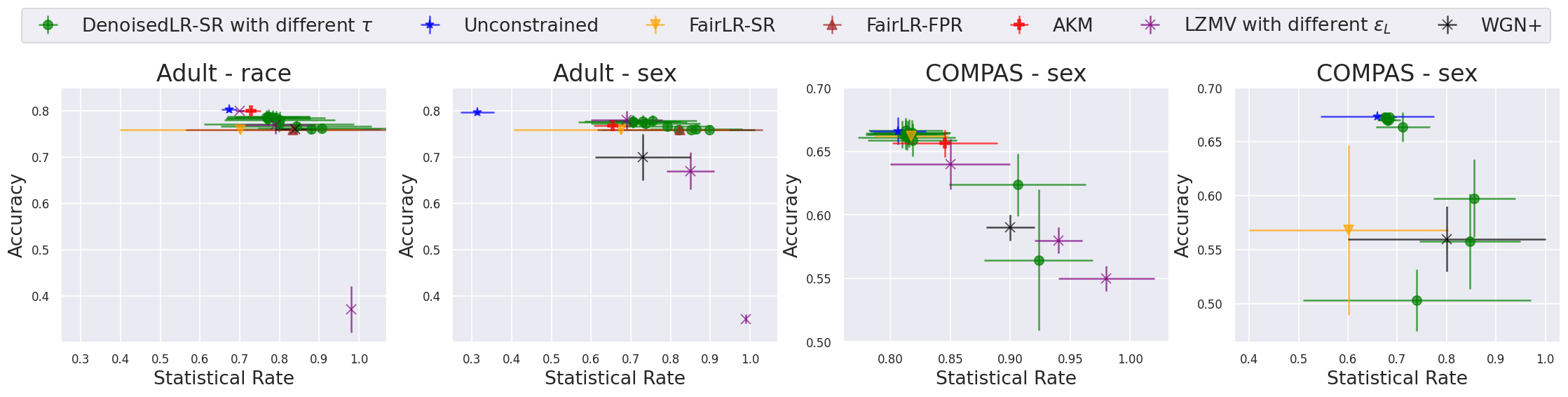} 
\caption{\small{Performance of \textbf{DLR-SR} and baselines with respect to statistical rate and accuracy for different combinations of dataset and protected attribute. For \textbf{DLR-SR}, the performance for different $\tau$ is presented, while
%to present the entire fairness-accuracy tradeoff picture. Similarly, 
for \textbf{LZMV} the input parameter $\epsilon_L$ is varied. The plots shows that for all settings  \textbf{DLR-SR} can attain a high statistical rate, often with minimal loss in accuracy.}}
\label{fig:stat_rate_plot}
\end{figure*}

\begin{figure*}
\centering
\includegraphics[width=\linewidth]{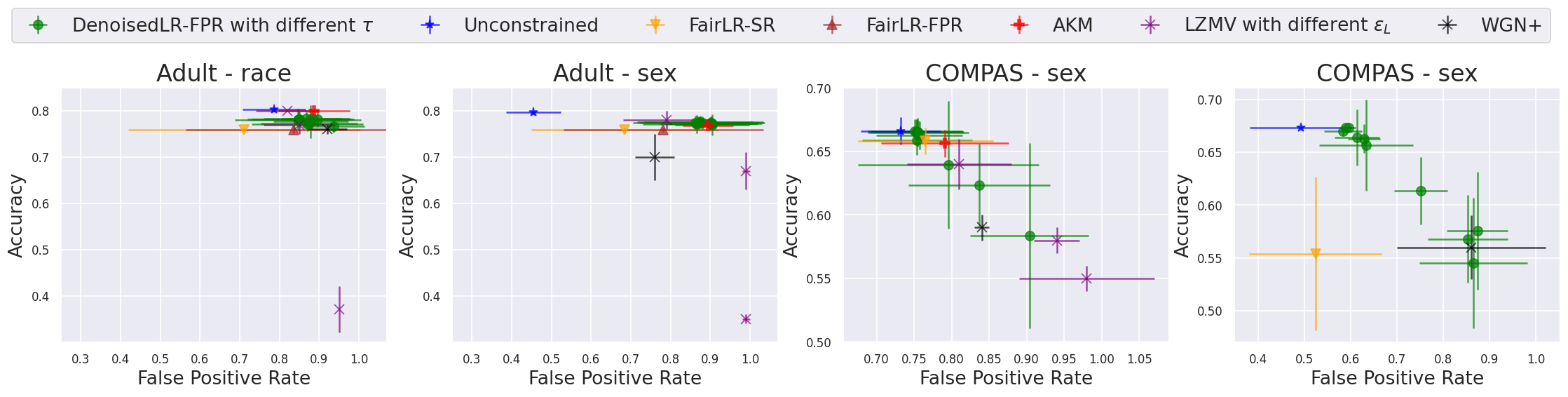} 
\caption{Performance of \textbf{DLR-FPR} and baselines with respect to false positive rate and accuracy for different combinations of dataset and protected attribute. For \textbf{DLR-FPR}, the performance for different $\tau$ is plotted to present the entire fairness-accuracy tradeoff picture. Similarly, for \textbf{LZMV} the input parameter $\epsilon_L$ is varied. The plots shows that for all settings  \textbf{FPR} can attain a high false positive rate, often with minimal loss in accuracy.}
\label{fig:fpr_plot}
\end{figure*}	
	
	\setlength\fboxsep{0pt}
    \setlength{\tabcolsep}{8pt}

	\begin{table*}[t]
		\centering
		\caption{\small The performance of all algorithms over test datasets with respect to false discovery rate fairness metric - average and standard error (in brackets) of accuracy and false discovery rate. 
		Our method \textbf{DLR-FDR}, with $\tau=0.9$, achieves higher false discovery rate than baselines in almost every setting, at a minimal cost to accuracy.
		}
		\label{tab:fdr_results}
		\scriptsize
		
		\begin{tabular}{p{2cm}cc|cc|cc|cc}
			\toprule
			\multirow{3}{*}{} & \multicolumn{4}{c|}{\textbf{Adult}} & \multicolumn{4}{c}{\textbf{COMPAS}} \\
			& \multicolumn{2}{c|}{sex (binary)} & \multicolumn{2}{c|}{race (binary)} &\multicolumn{2}{c|}{sex (binary)} &\multicolumn{2}{c}{race (non-binary)} \\
			& acc & FDR & acc  & FDR & acc  & FDR & acc & FDR \\
			\midrule
			\scriptsize{\textbf{LR-SR}} & .76 (.01) & .55 (.45) & .76 (.01) & .56 (.46) & .67 (.01) & .66 (0) & .58 (.05) & .73 (.06)\\
			\scriptsize{\textbf{LR-FPR}} & .76 (.01) & .54 (.45) & .76 (0) & .35 (.43) & .67 (.01)  & .75 (.09) &  .56 (.05) & .72 (.05)  \\
			\scriptsize{\specialcell{\textbf{LZMV} $\varepsilon_L=.01$}} & .35 (.01)   & 0 (0) & .37 (.05) & 0 (0)  & .55 (.01) & .74 (.04) & -  & - \\
			\scriptsize{\specialcell{\textbf{LZMV} $\varepsilon_L=.04$}} &  .67 (.04) & 0 (0)  & .77 (.03)  & 0 (0)  &  .58 (.01) & .74 (.04) & -   & -\\
			\scriptsize{\specialcell{\textbf{LZMV} $\varepsilon_L=.10$}} & .78 (.02)   & .47 (.01)  & .80 (0)  & .76 (.05)  & .64 (.02)  & .83 (.04) &  - & - \\
			\scriptsize{\textbf{AKM}} &  .77 (0) & .55 (.17)  &  {{.80} (0)} & .71 (.01) & .69 (.01) & .75 (.03) & -  & -  \\
			\scriptsize{\textbf{WGN+}} & .59 (0)  & .54 (.02)  &  .67 (0)   & .65 (.01)  &  .54 (.01)  & .72 (.05)  &  .56 (.03)  &  .68 (.07)  \\
			\midrule
			\scriptsize{\textbf{DLR-FDR} $\scriptsize \tau=.7$} &  .73 (.04) & .66 (.07) & .80 (.02) & .76 (.06) & .64 (.03) & .75 (.11) & .67 (.02) & .79 (.03)     \\
			\scriptsize{\textbf{DLR-FDR} $\scriptsize \tau=.9$} &  .75 (.01) & .87 (.08) & .76 (.02) & .89 (.09) & .60 (.07) & .77 (.10) & .54 (.13) & {.79} (.07)     \\
			\bottomrule
		\end{tabular}	
	\end{table*}

\section{Other empirical details and results}
	\label{sec:other}
	
	We state the exact empirical form of the constraints used for our simulations in this section and then present additional empirical results.
	
	\subsection{Implementation of our denoised algorithm.}
	%\textcolor{red}{Is it suitable to move here?}
	%\subsection{Algorithm for Program~\eqref{eq:progdenoised}}
	%\label{sec:algorithm}
	
	As a use case, 
	%and used in our empirical results, 
	we solve Program~\eqref{eq:progdenoised} for logistic regression.
	Let $\calF' = \left\{f'_\theta\mid \theta\in \R^d\right\}$  be the family of logistic regression classifiers where for each sample $s=(x,z,y)$,
	$
	f'_\theta(x):= \frac{1}{1+e^{-\langle x,\theta\rangle}}.
	$
	We learn a classifier $f'_\theta\in \calF'$ and then round each $f'_\theta(\widehat{x}_i)$ to $f_\theta(\widehat{x}_i) := \I\left[f(\widehat{x}_i)\geq  0.5\right]$.

	We next show how to implement the Program~\eqref{eq:progdenoised} for any general fairness constraints.
	Let $\xi(f)$ and $\xi'(f)$ denote the relevant events to measure the group performances. 
	The constraints use the group-conditional probabilities of these events, i.e.
	 $ \widehat{u}(f) := \left(\Pr\left[\xi(f), \xi'(f), \widehat{Z}=i\right]\right)_{i\in [p]} $ and $	\textstyle \widehat{w}(f) := \left(\Pr\left[\xi'(f), \widehat{Z}=i\right]\right)_{i\in [p]}$.
	Let $N = |S|$ and let $u'(f)$, $w'(f)$ denote the empirical approximation of $ \widehat{u}(f)$, $ \widehat{w}(f)$ respectively; i.e.,
	\[u'(f) :=  \left(\frac{1}{N} \sum_{\alpha \in [N], \hat{Z} = i}{\mathbf{1}\left[\xi(f(x_\alpha)), \xi'(f(x_\alpha))\right]} \right)_{i \in [p]},\]
	\[w'(f) :=  \left(\frac{1}{N} \sum_{\alpha \in [N], \hat{Z} = i}{\mathbf{1}\left[\xi'(f(x_\alpha))\right]} \right)_{i \in [p]}.\]
	Let $\Gamma_i'(f) := \left((H^\top)^{-1}u'(f)\right)_i/\left((H^\top)^{-1}w'(f)\right)_i$, for each $i \in [p]$ and $M:= \max_{i\in [p]} \|(H^\top)^{-1}_i\|_1$.
	Then, given $\tau \in [0,1]$ and $\lambda, \delta > 0$, the empirical implementation in Program~\eqref{eq:progdenoised} use the following constraints. 
	\begin{align}
	\label{eq:denoised_relaxed}
	\begin{cases}
    & \Gamma_i'(f) \geq (\tau - \delta) \cdot \Gamma_j'(f), \forall i, j \in [p] \times [p], \\
    & \left((H^\top)^{-1}u'(f)\right)_i \geq (\lambda - M\delta), \forall i \in [p].
    \end{cases}
    \end{align}
	
    The program \textbf{DLR} simply implements the following optimization problem.
	\begin{equation} \tag{DLR}
	\label{eq:prog_denoisedLR}
	\begin{split}
	&\min_{\theta\in \R^d} -\frac{1}{N}\sum_{a\in [N]} \left(y_a \log f_\theta(x_a) + (1-y_a) \log (1-f_\theta(x_a))\right) \\  &
	s.t. ~\text{Constraints (\eqref{eq:denoised_relaxed}) are satisfied}.
	\end{split}
	\end{equation}

	\paragraph{Program~\eqref{eq:progdenoised} for statistical rate metric (\textbf{DLR-SR}). }
	For statistical rate metric, simply set $\xi(f_\theta(x_\alpha)) = (f_\theta(x_\alpha)=1)$ and $\xi'(f_\theta(x_\alpha)) = \emptyset$, and compute the empirical constraints in Eqns~\ref{eq:denoised_relaxed}.
	
	\paragraph{Program~\eqref{eq:progdenoised} for false positive rate metric (\textbf{DLR-FPR}). }
	For false positive rate metric, set $\xi(f_\theta(x_\alpha)) = (f_\theta(x_\alpha)=1)$ and $\xi'(f_\theta(x_\alpha)) = (Y=0)$, and compute the empirical constraints in Eqns~\ref{eq:denoised_relaxed}.
	
	\paragraph{Program~\eqref{eq:progdenoised} for false discovery rate metric (\textbf{DLR-FDR}).}
	For false discovery rate metric, simply set $\xi(f_\theta(x_\alpha)) = (Y=0)$ and $\xi'(f_\theta(x_\alpha)) = (f_\theta(x_\alpha)=1)$, and compute the empirical constraints in Eqns~\ref{eq:denoised_relaxed}.
	
	\noindent
	If required, one can also append a regularization term $C\cdot \|\theta\|_2^2$ to the above loss function where $C\geq 0$ is a given regularization parameter.

	\subsection{SLSQP parameters}
	We use standard constrained optimization packages to solve this program, such as SLSQP~\cite{kraft1988software} (implemented using python \textit{scipy} package).
	For each optimization problem, we run the SLSQP algorithm for 500 iterations, starting with a randomly chosen point and with parameters ftol=1e-3 and eps=1e-3.
	
    \subsection{Baselines' parameters}
    \textbf{LZMV:} For this algorithm of \citet{lamy2019noise}, we use the implementation from \url{https://github.com/AIasd/noise_fairlearn} .
    The constraints are with respect to additive statistical rate.
    The fairness tolerance parameter $\varepsilon$ (referred to as $\varepsilon_L$ in our empirical results to avoid confusion) are chosen to be $\set{0.01, 0.04, 0.10}$ to present the range of performance of the algorithm.
    See the paper \cite{lamy2019noise} for descriptions of these parameters.
    The base classifier used is the algorithm of \citet{agarwal2018reductions}, and the noise parameters are provided as input to the \textbf{LZMV} algorithm.
    
    \textbf{AKM: } For this algorithm, we use the implementation from\\ \url{https://github.com/matthklein/equalized_odds_under_perturbation}.
    The constraints are with respect to additive false positive rate parity.
    Once again, the algorithm takes noise parameters as input and uses the base classifier of \citet{hardt2016equality}.
    
    \textbf{WGN+:} For this algorithm, we use the implementation from \url{https://github.com/wenshuoguo/robust-fairness-code}.
    Once again, the constraints here are additive false positive rate constraints using the soft-group assignments.
    See the paper \cite{wang2020robust} for descriptions of these parameters.
    The learning rate parameters used for this algorithm are $\eta_\theta \in \{.001, 0.01,0.1 \}$, $\eta_\lambda \in \{0.5,1.0, 2.0 \}$, and $\eta_W \in \{0.01,0.1 \}$. These parameters are same as the one the authors suggest in their paper and code.
    We run their algorithm for all combinations of the above parameters and select and report the test performance of the model that has the best training objective value, while satisfying the program constraints.

\begin{figure*}
\centering
\includegraphics[width=\linewidth]{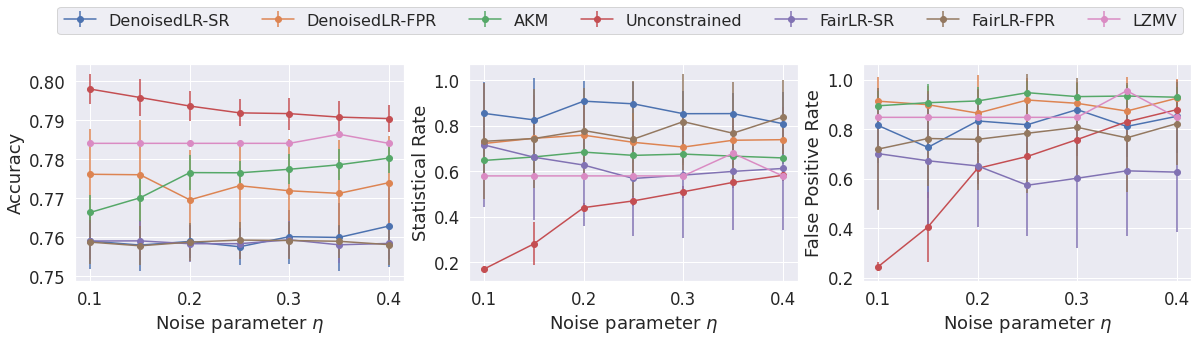} 
\subfloat[Accuracy vs $\eta$]{\hspace{.33\linewidth}}
\subfloat[Statistical Rate vs $\eta$]{\hspace{.33\linewidth}}
\subfloat[False Positive Rate vs $\eta$]{\hspace{.33\linewidth}}
\caption{Performance of \textbf{DLR-SR}, \textbf{DLR-FPR} ($\tau= 0.9$) and baselines with respect to statistical rate, false positive rate and accuracy for different noise parameters $\eta$. The dataset used is \textbf{Adult} and the protected attribute is sex.}
\label{fig:adult_sex_diff_eta}
\end{figure*}	
	
\begin{figure*}
\centering
\includegraphics[width=\linewidth]{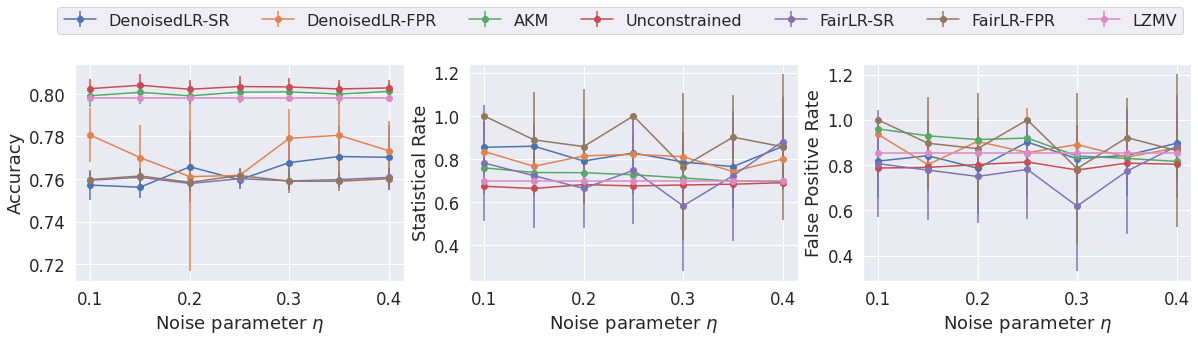} 
\subfloat[Accuracy vs $\eta$]{\hspace{.33\linewidth}}
\subfloat[Statistical Rate vs $\eta$]{\hspace{.33\linewidth}}
\subfloat[False Positive Rate vs $\eta$]{\hspace{.33\linewidth}}
\caption{Performance of \textbf{DLR-SR}, \textbf{DLR-FPR} ($\tau =  0.9$) and baselines with respect to statistical rate, false positive rate and accuracy for different noise parameters $\eta$. The dataset used is \textbf{Adult} and the protected attribute is race.}
\label{fig:adult_race_diff_eta}
\end{figure*}	
	
\begin{figure*}
\centering
\includegraphics[width=\linewidth]{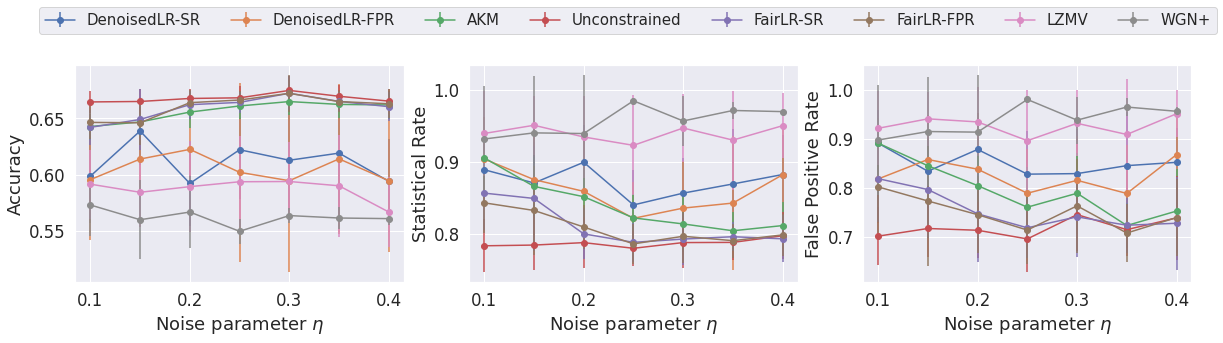} 
\subfloat[Accuracy vs $\eta$]{\hspace{.33\linewidth}}
\subfloat[Statistical Rate vs $\eta$]{\hspace{.33\linewidth}}
\subfloat[False Positive Rate vs $\eta$]{\hspace{.33\linewidth}}
\caption{Performance of \textbf{DLR-SR}, \textbf{DLR-FPR} ($\tau =  0.9$) and baselines with respect to statistical rate, false positive rate and accuracy for different noise parameters $\eta$. The dataset used is \textbf{COMPAS} and the protected attribute is sex.}
\label{fig:compas_sex_diff_eta}
\end{figure*}	
	
\begin{figure*}
\centering
\includegraphics[width=\linewidth]{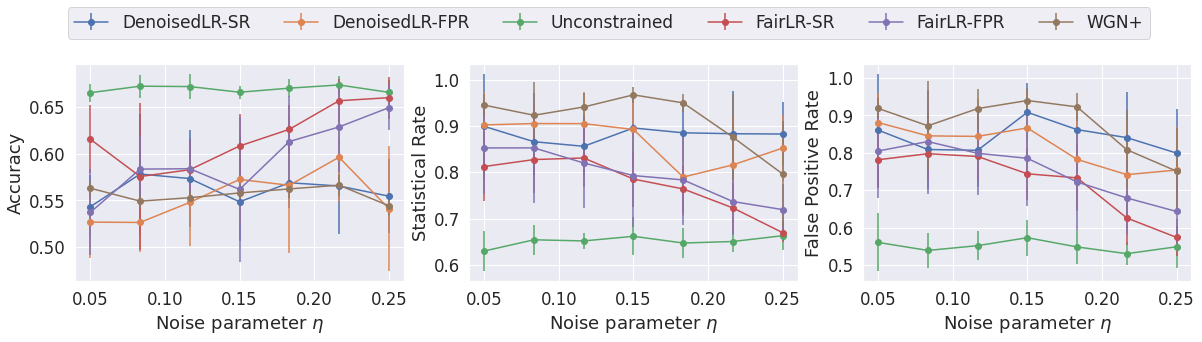} 
\subfloat[Accuracy vs $\eta$]{\hspace{.33\linewidth}}
\subfloat[Statistical Rate vs $\eta$]{\hspace{.33\linewidth}}
\subfloat[False Positive Rate vs $\eta$]{\hspace{.33\linewidth}}
\caption{Performance of \textbf{DLR-SR}, \textbf{DLR-FPR} ($\tau =  0.9$) and baselines with respect to statistical rate, false positive rate and accuracy for different noise parameters $\eta$. The dataset used is \textbf{COMPAS} and the protected attribute is race.}
\label{fig:compas_race_diff_eta}
\end{figure*}

	\subsection{Other results}
	In this section, we present other empirical results to complement the arguments made in Section~\ref{sec:empirical}.
	First, we present the plot for comparison of all methods with respect to statistical rate, Figure~\ref{fig:stat_rate_plot}, and false positive rate, Figure~\ref{fig:fpr_plot}.
	
	\subsubsection{Performance with respect to false discovery rate}
	We also present the empirical performance of our algorithm, compared to baselines, when the fairness metric in consideration is false discovery rate (a linear-fractional metric).
	Table~\ref{tab:fdr_results} presents the results.
	For most combinations of datasets and protected attributes, our method \textbf{DLR-FDR}, with $\tau=0.9$, achieves a higher false discovery rate than baselines, at a minimal cost to accuracy.
	
	\subsubsection{Variation of noise parameter}
	We also investigate the performances of algorithms w.r.t. varying $\eta_0, \eta_1$.
	We consider $\eta_0 = \eta_1 = \eta\in \left\{ 0.1,  0.15,  0.2,  0.25,  0.3,  0.35,  0.4\right\}$ for the binary case, and 
	$H_{i,j} \in \set{0.05, \cdots, 0.25}$, for $i \neq j$, in the non-binary case.
	Other settings are the same as in the main text.
	We select $\tau =  0.9$ for \textbf{FairLR} and \textbf{DLR}.
	The performance on Adult dataset is presented in Figure~\ref{fig:adult_sex_diff_eta} when sex is the protected attribute and in Figure~\ref{fig:adult_race_diff_eta} when race is the protected attribute.
	The performance on COMPAS dataset is presented in Figure~\ref{fig:compas_sex_diff_eta} when sex is the protected attribute and in Figure~\ref{fig:compas_race_diff_eta} when race is the protected attribute.

	\subsubsection{Error in noise parameter estimation}
	As discussed at the end of Section~\ref{sec:thm}, the scale of error in the noise parameter estimation can affect the fairness guarantees.
	In this section, we empirically look at the impact of estimation error on the statistical rate of the generated classifier.
	
	We set the true noise parameters $\eta_0 = \eta_1 =  0.3$. The estimated noise parameter ranges $\eta'$ ranges from  0.1 to  0.3. 
	The variation of accuracy and statistical rate with noise parameter estimate of \textbf{DenoisedLR-SR} for COMPAS and Adult datasets is presented in Figure~\ref{fig:all_noise_est}a,b.
    The plots show that, for both protected attributes, the best statistical rate (close to the desired guarantee of  0.90) is achieved when the estimate matches the true noise parameter value. 
    However, even for estimates that are considerably lower than the true estimate (for instance, $\eta' <  0.15$), the average statistical rate is still quite high ($\sim  0.80$). 

    The results show that if the error in the noise parameter estimate is reasonable, the framework ensures that the fairness of the generated classifier is still high.
    
    \begin{figure*}
    \centering
    \includegraphics[width=\linewidth]{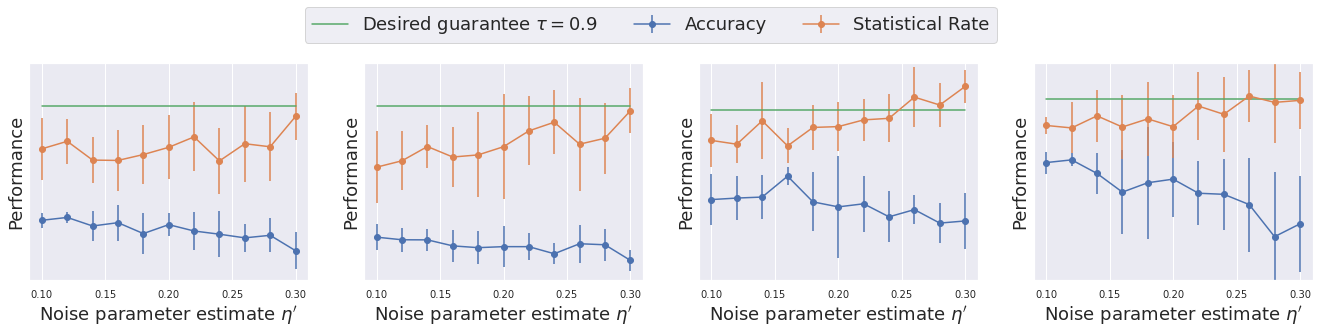} 
    \subfloat[COMPAS - sex]{\hspace{.25\linewidth}}
    \subfloat[COMPAS - race]{\hspace{.25\linewidth}}
    \subfloat[Adult - sex]{\hspace{.25\linewidth}}
    \subfloat[Adult - race]{\hspace{.25\linewidth}}
    \caption{Performance of \textbf{DLR-SR} ($\tau =  0.9$) with respect to statistical rate and accuracy for different noise parameter estimate $\eta'$. The true noise parameters are $\eta_0 = \eta_1 =  0.3$.}
    \label{fig:all_noise_est}
    \end{figure*}

\color{black}

\section{Discussion of initial attempts}
	\label{sec:discussion}
	
	We first discuss two natural ideas including randomized labeling (Section~\ref{sec:alg1}) and solving Program~\eqref{eq:progcon} that only depends on $\widehat{S}$ (Section~\ref{sec:alg2}). 
	For simplicity, we consider the same setting as in Section~\ref{sec:proof}: $p=2$ with statistical rate, and assume $\eta= \eta_1 = \eta_2\in (0, 0.4)$.
	We also discuss their weakness on either the empirical loss or the fairness constraints. 
	This section aims to show that directly applying the same fairness constraints on $\widehat{S}$ may introduce bias on $S$ and, hence, our modifications to the constraints (Definition~\ref{def:denoised}) are necessary.

	\subsection{Randomized labeling}
	\label{sec:alg1}
	
	A simple idea is that for each sample $s_a\in S$, i.i.d. draw the label $f(s_a)$ to be 0 with probability $\alpha$ and to be 1 with probability $1-\alpha$ ($\alpha\in [0,1]$).
	This simple idea leads to a fair classifier by the following lemma.

	\begin{lemma}[\bf{A random classifier is fair}]
		\label{lm:random}
		Let $f\in \left\{0,1\right\}^\calX$ be a classifier generated by randomized labeling.
		With probability at least $1-2e^{-\frac{\alpha \lambda N}{1.2\times10^5}}$, $\gamma(f,S)\geq  0.99$.
	\end{lemma}
	
	\begin{proof}
		Let $A = \left\{a\in [N]: z_a = 0 \right\}$ be the collection of samples with $Z=0$.
		By Assumption~\ref{assumption:ratio}, we know that $|A|\geq \lambda N$.
		For $a\in A$, let $X_a$ be the random variable where $X_a=f(s_a)$.
		By randomized labeling, we know that $\Pr\left[X_i = 1\right] = \alpha$.
		Also, 
		\begin{eqnarray}
		\label{eq:random1}
		\Pr\left[f=1\mid Z=0\right] = \frac{\sum_{i\in A} X_i}{|A|}.
		\end{eqnarray}
		Since all $X_i$ ($i\in A$) are independent, we have
		\begin{eqnarray}
		\label{eq:random2}
		\begin{split}
		 \Pr\left[\sum_{i\in A} X_i\in (1\pm  0.005)\cdot \alpha |A| \right] 
		&\geq && 1-2e^{-\frac{ 0.005^2 \alpha |A|}{3}} \quad (\text{Chernoff bound}) \\
		& \geq && 1-2e^{-\frac{\alpha \lambda N}{1.2\times 10^5}}. \quad (|A|\geq \lambda N)
		\end{split}
		\end{eqnarray}
		Thus, with probability at least $1-2e^{-\frac{\alpha \lambda N}{1.2\times 10^5}}$,
		\begin{eqnarray*}
			\label{eq:random3}
			\begin{split}
				\Pr\left[f=1\mid Z=0\right]
				& = && \frac{\sum_{i\in A} X_i}{|A|} & (\text{Eq.~\eqref{eq:random1}})\\
				& \in && (1\pm  0.005)\cdot \frac{\alpha |A|}{|A|} & (\text{Ineq.~\eqref{eq:random2}}) \\
				& \in && (1\pm  0.005) \alpha. &
			\end{split}
		\end{eqnarray*}
		Similarly, we have that with probability at least $1-2e^{-\frac{\alpha \lambda N}{1.2\times 10^5}}$,
		\[
		\Pr\left[f=1\mid Z=1\right]\in (1\pm  0.005) \alpha.
		\] 
		By the definition of $\gamma(f,S)$, we complete the proof.
	\end{proof}
	
	\noindent
	However, there is no guarantee for the empirical risk of randomized labeling.
	For instance, consider the loss function $L(f,s):= \I\left[f(s)=y\right]$ where $\I\left[\cdot\right]$ is the indicator function, and suppose there are $\frac{N}{2}$ samples with $y_a=0$.
	In this setting, the empirical risk of $f^\star$ may be close to 0, e.g., $f^\star = Y$.
	Meanwhile, the expected empirical risk of randomized labeling is
	\[
	\frac{1}{N} \left((1-\alpha)\cdot \frac{N}{2} +\alpha\cdot \frac{N}{2}\right) = \frac{1}{2},
	\]
	which is much larger than that of $f^\star$.

	\subsection{Replacing $S$ by $\widehat{S}$ in Program~\eqref{eq:progtarget}}
	\label{sec:alg2}
	
	Another idea is to solve the following program which only depends on $\widehat{S}$, i.e., simply replacing $S$ by $\widehat{S}$ in Program~\eqref{eq:progtarget}. 
	
	\begin{tcolorbox}
		\begin{equation} \tag{ConFair}
		\label{eq:progcon}
		\begin{split}
		& \min_{f\in \calF} \frac{1}{N}\sum_{a\in [N]} L(f, \widehat{s}_a) \quad s.t. \\
		& ~ \gamma(f, \widehat{S})\geq \tau.
		\end{split}
		\end{equation}
	\end{tcolorbox}
	
	\begin{remark}
		\label{remark:5}
		Similar to Section~\ref{sec:empirical}, we can design an algorithm that solves Program~\eqref{eq:progcon} by logistic regression.
		\begin{equation} \tag{FairLR}
		\small
		\label{eq:progLR}
		\begin{split}
		&\min_{\theta\in \R^d} -\frac{1}{N}\sum_{a\in [N]} \left(y_a \log f_\theta(s_a) + (1-y_a) \log (1-f_\theta(s_a))\right) \\ 
		s.t.  &~\widehat{\mu}_1\cdot \sum_{a\in [N]: \widehat{Z}=0} \I\left[\langle x_a, \theta\rangle \geq 0\right]
		\geq \tau \widehat{\mu}_0\cdot \sum_{a\in [N]: \widehat{Z}=1} \I\left[\langle x_a, \theta\rangle \geq 0\right], \\
		&~\widehat{\mu}_0\cdot \sum_{a\in [N]: \widehat{Z}=1} \I\left[\langle x_a, \theta\rangle \geq 0\right] 
		\geq  \tau \widehat{\mu}_1\cdot \sum_{a\in [N]: \widehat{Z}=0} \I\left[\langle x_a, \theta\rangle \geq 0\right].
		\end{split}
		\end{equation}
	\end{remark}
	
	\noindent
	Let $\widehat{f}^\star$ denote an optimal solution of Program~\eqref{eq:progcon}.
	Ideally, we want to use $\widehat{f}^\star$ to estimate $f^\star$.
	Since $Z$ is not used for prediction, we have that for any $f\in \calF$,
	\[
	\sum_{a\in [N]} L(f,s_a) = \sum_{a\in [N]} L(f,\widehat{s}_a).
	\]
	Then if $\widehat{f}^\star$ satisfies $\gamma(\widehat{f}^\star, S)\geq \tau$, we conclude that $\widehat{f}^\star$ is also an optimal solution of Program~\eqref{eq:progtarget}. 
	However, due to the flipping noises, $\widehat{f}^\star$ may be far from $f^\star$ (Example~\ref{example:1}).
	More concretely, it is possible that $\gamma(\widehat{f}^\star,S) \ll \tau$ (Lemma~\ref{lm:gap1}).
	Moreover, we discuss the range of $\Omega(f^\star, \widehat{S})$ (Lemma~\ref{lm:gap2}). 
	We find that $\Omega(f^\star, \widehat{S}) < \tau$ may hold which implies that $f^\star$ may not be feasible for Program~\eqref{eq:progcon}.
	We first give an example showing that $\widehat{f}^\star$ can perform very bad over $S$ with respect to the fairness metric.

	\begin{figure*}
		\caption{An example showing that $\gamma(f,S)$ and $\gamma(f,\widehat{S})$ can differ by a lot. The detailed explanation can be found in Example~\ref{example:1}.}
		\begin{center}
			\includegraphics[width = \textwidth ]{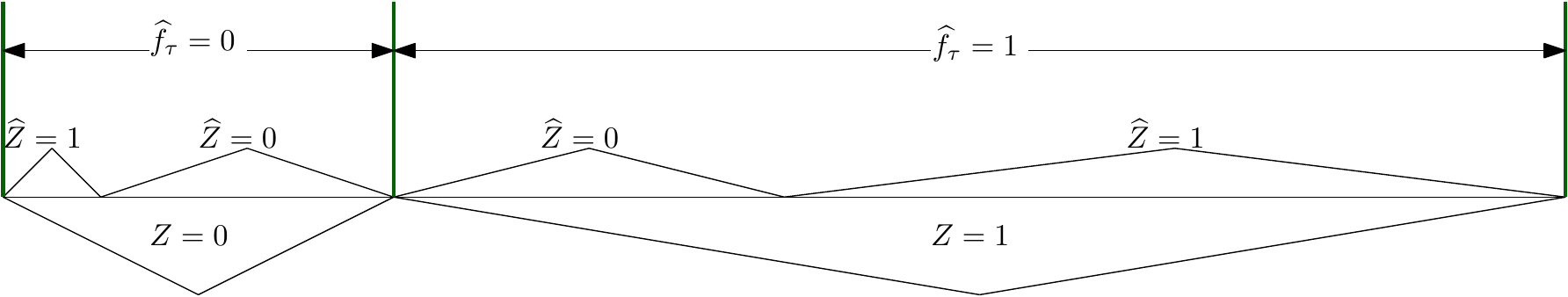}
		\end{center}
		\label{fig:worstcase}
	\end{figure*}
	
	\begin{example}
		\label{example:1}
		Our example is shown in Figure~\ref{fig:worstcase}.
		We assume that $\mu_0 = 1/3$ and $\mu_1 = 2/3$.
		Let $\eta = 1/3$ be the noise parameter and we assume
		$
		\pi_{20} = \pi_{01} = 1/3.
		$
		Consequently, we have that 
		\[
		\widehat{\mu}_0 = 1/3\times 2/3 + 2/3 *1/3 = 4/9.
		\]
		\sloppy
		Then we consider the following simple classifier $f\in \left\{0,1\right\}^\calX$: $\widehat{f}^\star=Z$.
		We directly have that $\Pr\left[\widehat{f}^\star=1\mid Z=0\right] = 0$ and  $\Pr\left[\widehat{f}^\star=1\mid Z=1\right] = 1$, which implies that $\gamma(\widehat{f}^\star,S) = 0$.
		We also have that
		\begin{align*}
		 \Pr\left[\widehat{f}^\star=1\mid \widehat{Z}=0\right] 
		&=  \Pr\left[Z=1\mid \widehat{Z}=0\right] \quad (\widehat{f}^\star=Z) \\
		&= \frac{\pi_{01}\cdot\mu_1 }{\widehat{\mu}_0} \quad (\text{Observation~\ref{observation:prob}}) = 0.5, 
		\end{align*}
		and
		\begin{align*}
		\Pr\left[\widehat{f}^\star=1\mid \widehat{Z}=1\right]
		&= \Pr\left[Z=1\mid \widehat{Z}=1\right] \quad (\widehat{f}^\star=Z) \\
		&= \frac{\pi_{11}\cdot\mu_1 }{\widehat{\mu}_1} \quad (\text{Observation~\ref{observation:prob}}) =   0.8, &
		\end{align*}
		which implies that $\gamma(\widehat{f}^\star,\widehat{S}) =  0.625$.
		Hence, there is a gap between $\gamma(\widehat{f}^\star,S)$ and $\gamma(\widehat{f}^\star,\widehat{S})$, say  0.625, in this example.
		Consequently, $\widehat{f}^\star$ can be very unfair over $S$, and hence, is far from $f^\star$.
	\end{example}
	
	\noindent
	Next, we give some theoretical results showing the weaknesses of Program~\eqref{eq:progcon}.

	\paragraph{An upper bound for $\gamma(f, S)$.}
	More generally, given a classifier $f\in \left\{0,1\right\}^\calX$, we provide an upper bound for $\gamma(f,S)$ that is represented by $\gamma(f,\widehat{S})$; see the following lemma.

	\begin{lemma}[\bf{An upper bound for $\gamma(f,S)$}]
		\label{lm:gap1}
		Suppose we have
		\begin{enumerate}
			\item $\Pr\left[f=1\mid \widehat{Z}= 0\right] \leq \Pr\left[f=1\mid \widehat{Z}=1\right]$;
			\item $\Pr\left[f=1, Z=0\mid \widehat{Z}=0\right]\leq \alpha_0\cdot \Pr\left[f=1, Z=1\mid \widehat{Z}=0\right]$ for some $\alpha_0\in [0,1]$;
			\item $\Pr\left[f=1, Z=0\mid \widehat{Z}=1 \right]\leq \alpha_1\cdot \Pr\left[f=1, Z=1\mid \widehat{Z}=1 \right]$ for some $\alpha_1\in [0,1]$.
		\end{enumerate}
		Let $\beta_{ij} = \frac{\widehat{\mu}_i}{\mu_j}$ for $i,j\in \left\{0,1\right\}$.
		The following inequality holds
		\begin{eqnarray*}
		\gamma(f,S)\leq \frac{\alpha_0(1+\alpha_1) \beta_{00}\cdot \gamma(f,\widehat{S}) + \alpha_1(1+\alpha_0)\beta_{10}}{(1+\alpha_1) \beta_{01}\cdot \gamma(f,\widehat{S}) + (1+\alpha_0)\beta_{11}}
		\leq \max\left\{\alpha_0, \alpha_1\right\}\cdot \frac{\mu_1}{\mu_0}.
		\end{eqnarray*}
	\end{lemma}

	\noindent
	The intuition of the first assumption is that the statistical rate for $Z=0$ is at most that for $Z=1$ over the noisy dataset $\widehat{S}$.
	The second and the third assumptions require the classifier $f$ to be less positive when $Z=0$.
	Intuitively, $f$ is restricted to induce a smaller statistical rate for $Z=0$ over both $S$ and $\widehat{S}$.
	Specifically, if $\alpha_0=\alpha_1=0$ as in Example~\ref{example:1}, we have $\gamma(f,S)=0$.
	Even if $\alpha_0 = \alpha_1 = 1$, we have $\gamma(f,S)\leq \frac{\mu_1}{\mu_0}$ which does not depend on $\gamma(f, \widehat{S})$. 
	\begin{proof}[Proof of Lemma~\ref{lm:gap1}]
		By the first assumption, we have
		\begin{eqnarray}
		\label{eq:lower1}
		\gamma(f,\widehat{S}) = \frac{\Pr\left[f=1\mid \widehat{Z}= 0\right]}{\Pr\left[f=1\mid \widehat{Z}= 1\right]}.
		\end{eqnarray}
		By the second assumption, we have 
		\begin{eqnarray}
		\label{eq:lower2}
		\begin{split}
		\Pr\left[f=1, Z=1\mid \widehat{Z}=0\right]
		&= && \frac{(1+\alpha_0)\cdot\Pr\left[f=1, Z=1\mid \widehat{Z}=0\right]}{1+\alpha_0} & \\
		& \geq && \frac{\Pr\left[f=1, Z=1\mid \widehat{Z}=0\right]}{1+\alpha_0} + \frac{\Pr\left[f=1, Z=0\mid \widehat{Z}=0\right]}{1+\alpha_0} &\\
		& = && \frac{1}{1+\alpha_0}\cdot \Pr\left[f=1\mid \widehat{Z}=0\right].
		\end{split}
		\end{eqnarray}
		Similarly, we have the following
		\begin{eqnarray}
		\label{eq:lower3}
		\begin{split}
		\Pr\left[f=1, Z=0\mid \widehat{Z}=0\right] \leq \frac{\alpha_0}{1+\alpha_0} \Pr\left[f=1\mid \widehat{Z}=0\right].
		\end{split}
		\end{eqnarray}
		Also, by the third assumption, we have
		\begin{eqnarray}
		\label{eq:lower4}
		\begin{split}
		\Pr\left[f=1, Z=1\mid \widehat{Z}=1\right] 
		\geq \frac{1}{1+\alpha_1} \Pr\left[f=1\mid \widehat{Z}=1\right],
		\end{split}
		\end{eqnarray}
		and
		\begin{eqnarray}
		\label{eq:lower5}
		\begin{split}
		\Pr\left[f=1, Z=0\mid \widehat{Z}=1\right]
		\leq \frac{\alpha_1}{1+\alpha_1} \Pr\left[f=1\mid \widehat{Z}=1\right].
		\end{split}
		\end{eqnarray}
		Then
		\begin{eqnarray}
		\label{eq:lower6}
		\begin{split}
		\Pr\left[f=1 \mid Z=0\right] & = && \Pr\left[f=1,\widehat{Z}=0\mid Z=0\right] + \Pr\left[f=1,\widehat{Z}=1 \mid Z=0\right]  \\
		& = && \Pr\left[f=1,Z=0\mid \widehat{Z}=0\right]\cdot \frac{\widehat{\mu}_0}{\mu_0} + \Pr\left[f=1,Z=0\mid \widehat{Z}=1\right]\cdot \frac{\widehat{\mu}_1}{\mu_0}  \\
		& = && \Pr\left[f=1,Z=0\mid \widehat{Z}=0\right]\cdot \beta_{00}  + \Pr\left[f=1,Z=0\mid \widehat{Z}=1\right]\cdot \beta_{10} \\
		& && (\text{Defn. of $\beta_{00}$ and $\beta_{10}$}) \\
		& \leq && \frac{\alpha_0\beta_{00}}{1+\alpha_0}\cdot \Pr\left[f=1\mid \widehat{Z}=0\right]  + \frac{\alpha_1 \beta_{10}}{1+\alpha_1}\cdot \Pr\left[f=1\mid \widehat{Z}=1\right]. \\
		& && (\text{Ineqs.~\eqref{eq:lower3} and~\eqref{eq:lower5}})
		\end{split}
		\end{eqnarray}
		By a similar argument, we have
		\begin{eqnarray}
		\label{eq:lower7}
		\begin{split}
		\Pr\left[f=1 \mid Z=1\right]
		& = && \Pr\left[f=1,Z=1\mid \widehat{Z}=0\right]\cdot \beta_{01}  + \Pr\left[f=1,Z=1\mid \widehat{Z}=1\right]\cdot \beta_{11} \\
		& && (\text{Defn. of $\beta_{01}$ and $\beta_{11}$}) \\
		& \geq && \frac{\beta_{01}}{1+\alpha_0}\cdot \Pr\left[f=1\mid \widehat{Z}=0\right]  + \frac{\beta_{11}}{1+\alpha_1}\cdot \Pr\left[f=1\mid \widehat{Z}=1\right]. \\
		& && (\text{Ineqs.~\eqref{eq:lower2} and~\eqref{eq:lower4}})
		\end{split}
		\end{eqnarray}
		Thus, we have
		\begin{eqnarray*}
		\small
			\begin{split}
				 \gamma(f,S) 
				& \leq && \frac{\Pr\left[f=1 \mid Z=0\right]}{ \Pr\left[f=1 \mid Z=1\right]} \quad (\text{Defn. of $\gamma(f,S)$}) \\
				& \leq && \frac{\frac{\alpha_0\beta_{00}}{1+\alpha_0} \Pr\left[f=1\mid \widehat{Z}=0\right] + \frac{\alpha_1 \beta_{10}}{1+\alpha_1} \Pr\left[f=1\mid \widehat{Z}=1\right]}{\frac{\beta_{01}}{1+\alpha_0} \Pr\left[f=1\mid \widehat{Z}=0\right] + \frac{\beta_{11}}{1+\alpha_1} \Pr\left[f=1\mid \widehat{Z}=1\right]} \\
				& && (\text{Ineqs.~\eqref{eq:lower6} and~\eqref{eq:lower7}}) \\
				& = && \frac{\alpha_0(1+\alpha_1) \beta_{00}\cdot \gamma(f,\widehat{S}) + \alpha_1(1+\alpha_0)\beta_{10}}{(1+\alpha_1) \beta_{01}\cdot \gamma(f,\widehat{S}) + (1+\alpha_0)\beta_{11}} \quad (\text{Eq.~\eqref{eq:lower1}}) \\
				& \leq && \max\left\{\alpha_0\cdot \frac{\beta_{00}}{\beta_{01}}, \alpha_1\cdot \frac{\beta_{10}}{\beta_{11}}\right\} \\
				& = && \max\left\{\alpha_0, \alpha_1\right\}\cdot \frac{\mu_1}{\mu_0}, \quad (\text{Defn. of $\beta_{ij}$})
			\end{split}
		\end{eqnarray*}
		which completes the proof.
	\end{proof}
	
	\paragraph{$f^\star$ may not be feasible in Program~\eqref{eq:progcon}.}
	We consider a simple case that $\eta_1 = \eta_2 = \eta$.
	Without loss of generality, we assume that $\Pr\left[f^\star=1\mid Z = 0\right]\leq \Pr\left[f^\star=1\mid Z = 1\right]$, i.e., the statistical rate of $Z=0$ is smaller than that of $Z=1$ over $S$.
	Consequently, we have
	\[
	\gamma(f^\star, S) = \frac{\Pr\left[f^\star=1\mid Z = 0\right]}{\Pr\left[f^\star=1\mid Z = 1\right]}.
	\]
	
	\begin{lemma}[\bf{Range of $\Omega(f^\star, \widehat{S})$}]
		\label{lm:gap2}
		Let $\eps\in (0, 0.5)$ be a given constant and let
		\begin{align*}
		\small
		\Gamma = \frac{\eta\mu_0 + (1-\eta) (1-\mu_0)}{(1-\eta)\mu_0 + \eta (1-\mu_0)}\times \frac{(1-\eta)\mu_0 \gamma(f^\star, S)+\eta(1-\mu_0)}{\eta\mu_0 \gamma(f^\star,S)+(1-\eta)(1-\mu_0)}.
		\end{align*}
		With probability at least $1-4e^{-\frac{\eps^2 \eta \lambda N}{192}}$, the following holds
		\begin{eqnarray*}
			\gamma(f^\star, \widehat{S}) \in (1\pm \eps)\cdot \min\left\{\Gamma, \frac{1}{\Gamma} \right\}.
		\end{eqnarray*}
	\end{lemma}
	
	\noindent
	For instance, if $\mu_0= 0.5$, $\gamma(f^\star,S)= 0.8=\tau$ and $\eta =  0.2$, we have
	\[
	\gamma(f^\star, \widehat{S}) \approx  0.69 < \tau.
	\]
	Then $f^\star$ is not a feasible solution of Program~\eqref{eq:progcon}.
	Before proving the lemma, we give some intuitions.
	
	\begin{discussion}
		\label{discussion:1}
		By assumption, we have that for a given classifier $f^\star\in \calF$, 
		\begin{eqnarray}
		\label{eq:discussion1}
		\Pr\left[\widehat{Z}=1\mid Z=0\right] \approx \Pr\left[\widehat{Z}=0\mid Z=1\right] \approx \eta
		\end{eqnarray}
		Moreover, the above property also holds when conditioned on a subset of samples with $Z=0$ or $Z=1$.
		Specifically, for $i\in \left\{0,1\right\}$,
		\begin{eqnarray}
		\label{eq:discussion2}
		\begin{split}
		& && \Pr\left[\widehat{Z}=1\mid f^\star=1,Z=0\right] \\
		&\approx &&\Pr\left[\widehat{Z}=0\mid f^\star=1,Z=1\right] 		\approx \eta
		\end{split}
		\end{eqnarray}
		Another consequence of Property~\eqref{eq:discussion1} is that for $i\in \left\{0,1\right\}$,
		\begin{eqnarray}
		\label{eq:discussion3}
		\begin{split}
		\widehat{\mu}_i &= && \pi_{i,i}\mu_i + \pi_{i,1-i}\mu_{1-i} & (\text{Observation~\ref{observation:prob}}) \\
		&\approx && (1-\eta)\mu_i + \eta\mu_{1-i}. & (\text{Property~\eqref{eq:discussion1}})
		\end{split}
		\end{eqnarray}

		Then we have
		\begin{eqnarray*}
		\small
			\begin{split}
			    & && \Pr\left[f^\star=1\mid \widehat{Z}=0\right] \\
				& = && \Pr\left[f^\star=1, Z = 0\mid \widehat{Z}=0\right]  +  \Pr\left[f^\star=1, Z = 1\mid \widehat{Z}=0\right] \\
				& = &&  \Pr\left[Z = 0\mid \widehat{Z}=0\right]\cdot  \Pr\left[f^\star=1\mid Z = 0, \widehat{Z}=0\right]  + \Pr\left[Z = 1\mid \widehat{Z}=0\right]\cdot  \Pr\left[f^\star=1\mid Z = 1, \widehat{Z}=0\right] \\
				& = &&  \frac{\pi_{00} \mu_0}{\widehat{\mu}_0}\cdot \Pr\left[f^\star=1\mid Z = 0, \widehat{Z}=0\right]  + \frac{\pi_{01} \mu_1}{\widehat{\mu}_0}\cdot \Pr\left[f^\star=1\mid Z = 1, \widehat{Z}=0\right] \\
				& && (\text{Observation~\ref{observation:prob}})\\
				& \approx && \frac{(1-\eta) \mu_0}{(1-\eta)\mu_0 + \eta \mu_1}\cdot \Pr\left[f^\star=1\mid Z = 0, \widehat{Z}=0\right] + \frac{\eta \mu_1}{(1-\eta)\mu_0 + \eta \mu_1}  \cdot \Pr\left[f^\star=1\mid Z = 1, \widehat{Z}=0\right]  \\ 
				& && (\text{Properties~\eqref{eq:discussion1} and~\eqref{eq:discussion3}}) \\
				& = && \frac{(1-\eta) \mu_0}{(1-\eta)\mu_0 + \eta (1-\mu_0)}\times \\  
				& && \frac{\Pr\left[f^\star=1\mid Z=0\right]\cdot \Pr\left[\widehat{Z}=0\mid f^\star=1,Z=0\right]}{\Pr\left[\widehat{Z}=0\mid Z=0\right]}   + \frac{\eta \mu_1}{(1-\eta)\mu_0 + \eta (1-\mu_0)}\times \\   
				& && \frac{\Pr\left[f^\star=1\mid Z=1\right]\cdot \Pr\left[\widehat{Z}=0\mid f^\star=1,Z=1\right]}{\Pr\left[\widehat{Z}=0\mid Z=1\right]}  \\
				& \approx && \frac{(1-\eta) \mu_0}{(1-\eta)\mu_0 + \eta (1-\mu_0)}\cdot \Pr\left[f^\star=1\mid Z=0\right] + \frac{\eta \mu_1}{(1-\eta)\mu_0 + \eta (1-\mu_0)}\cdot   \Pr\left[f^\star=1\mid Z=1\right]. \\ 
				& && (\text{Properties~\eqref{eq:discussion1} and~\eqref{eq:discussion2}})
			\end{split}
		\end{eqnarray*}
		Similarly, we can represent 
		\begin{eqnarray*}
			\begin{split}
				& &&\Pr\left[f^\star=1\mid \widehat{Z}=1\right] \\
				&\approx &&\frac{\eta \mu_0  }{\eta\mu_0 + (1-\eta) (1-\mu_0)} \Pr\left[f^\star=1\mid Z = 0\right]  + \frac{(1-\eta) \mu_1}{\eta\mu_0 + (1-\eta) (1-\mu_0)} \Pr\left[f^\star=1\mid Z = 1\right].
			\end{split}
		\end{eqnarray*}
		Applying the approximate values of $\Pr\left[f^\star=1\mid \widehat{Z}=0\right]$ and $\Pr\left[f^\star=1\mid \widehat{Z}=1\right]$ to compute $\gamma(f^\star,S)$, we have Lemma~\ref{lm:gap2}.
	\end{discussion}
	
	\begin{proof}[Proof of Lemma~\ref{lm:gap2}]
		By definition, we have
		\[
		\gamma(f^\star,\widehat{S}) \leq \frac{\Pr\left[f^\star=1\mid \widehat{Z}=0\right]}{\Pr\left[f^\star=1\mid \widehat{Z}=1\right]}.
		\]
		\sloppy
		Thus, it suffices to provide an upper bound for $\Pr\left[f^\star=1\mid \widehat{Z}=0\right]$ and a lower bound for $\Pr\left[f^\star=1\mid \widehat{Z}=1\right]$.
		Similar to Discussion~\ref{discussion:1}, we have
		\begin{eqnarray}
		\label{eq:gap2_1}
		\begin{split}
		Pr\left[f^\star=1\mid \widehat{Z}=0\right] 
		&=&& \frac{\Pr\left[Z=0\right]\cdot \Pr\left[f^\star=1\mid Z=0\right]}{\Pr\left[\widehat{Z}=0\right]} \times  \Pr\left[\widehat{Z}=0\mid f^\star=1,Z=0\right]\\  &+ && \frac{\Pr\left[Z=1\right]\cdot \Pr\left[f^\star=1\mid Z=1\right]}{\Pr\left[\widehat{Z}=0\right]}\times  \Pr\left[\widehat{Z}=0\mid f^\star=1,Z=1\right]\\
		& = && \frac{\mu_0\cdot \Pr\left[f^\star=1\mid Z=0\right]}{\pi_{00}\mu_0 + \pi_{01}(1-\mu_0)} \times  \Pr\left[\widehat{Z}=0\mid f^\star=1,Z=0\right]\\  &+&& \frac{\mu_1\cdot \Pr\left[f^\star=1\mid Z=1\right]}{\pi_{00}\mu_0 + \pi_{01}(1-\mu_0)} \times  \Pr\left[\widehat{Z}=0\mid f^\star=1,Z=1\right],
		\end{split}
		\end{eqnarray}
		and
		\begin{eqnarray}
		\label{eq:gap2_2}
		\begin{split}
		\Pr\left[f^\star=1\mid \widehat{Z}=1\right] 
		&=&& \frac{\Pr\left[Z=0\right]\cdot \Pr\left[f^\star=1\mid Z=0\right]}{\Pr\left[\widehat{Z}=1\right]} \times  \Pr\left[\widehat{Z}=1\mid f^\star=1,Z=0\right] \\
		& && + \frac{\Pr\left[Z=1\right]\cdot \Pr\left[f^\star=1\mid Z=1\right]}{\Pr\left[\widehat{Z}=1\right]} \times \Pr\left[\widehat{Z}=1\mid f^\star=1,Z=1\right]\\
		& = && \frac{\mu_0\cdot \Pr\left[f^\star=1\mid Z=0\right]}{\pi_{11}(1-\mu_0) + \pi_{20}\mu_0}\times   \Pr\left[\widehat{Z}=1\mid f^\star=1,Z=0\right] \\
		& && + \frac{\mu_1\cdot \Pr\left[f^\star=1\mid Z=1\right]}{\pi_{11}(1-\mu_0) + \pi_{20}\mu_0} \times  \Pr\left[\widehat{Z}=1\mid f^\star=1,Z=1\right],
		\end{split}
		\end{eqnarray}
		We then analyze the right side of the Equation~\eqref{eq:gap2_1}.
		We take the term $\Pr\left[\widehat{Z}=0\mid f^\star=1,Z=1\right]$ as an example.
		Let $A=\left\{a\in [N]: f^\star(s_a) = 1, z_a = 0\right\}$.
		By Assumption~\ref{assumption:ratio}, we have $|A|\geq \lambda N$.
		For $i\in A$, let $X_i$ be the random variable where $X_i=1-\widehat{z}_i$.
		By Definition~\ref{def:flippingnoise}, we know that $\Pr\left[X_i = 1\right] = \eta$.
		Also, 
		\begin{eqnarray}
		\label{eq:gap2_3}
		\Pr\left[\widehat{Z}=0\mid f^\star=1,Z=1\right] = \frac{\sum_{i\in A} X_i}{|A|}.
		\end{eqnarray}
		Since all $X_i$ ($i\in A$) are independent, we have
		\begin{eqnarray}
		\label{eq:gap2_4}
		\begin{split}
		\Pr\left[\sum_{i\in A} X_i\in (1\pm \frac{\eps}{8})\cdot\eta |A| \right] 
		&\geq && 1-2e^{-\frac{\eps^2 \eta |A|}{192}} \quad (\text{Chernoff bound}) \\
		& \geq && 1-2e^{-\frac{\eps^2 \eta \lambda N}{192}}. \quad (|A|\geq \lambda N)
		\end{split}
		\end{eqnarray}
		Thus, with probability at least $1-2e^{-\frac{\eps^2 \eta \lambda N}{192}}$,
		\begin{eqnarray*}
			\label{eq:gap2_5}
			\begin{split}
				\Pr\left[\widehat{Z}=0\mid f^\star=1,Z=1\right] 
				& = && \frac{\sum_{i\in A} X_i}{|A|} & (\text{Eq.~\eqref{eq:gap2_3}})\\
				& \in && (1\pm \frac{\eps}{8})\cdot \frac{\eta |A|}{|A|} & (\text{Ineq.~\eqref{eq:gap2_4}}) \\
				& \in && (1\pm \frac{\eps}{8}) \eta. &
			\end{split}
		\end{eqnarray*}
		Consequently, we have
		\begin{align*}
		\Pr\left[\widehat{Z}=1\mid f^\star=1,Z=1\right]
		& = && 1- \Pr\left[\widehat{Z}=0\mid f^\star=1,Z=1\right] & \\
		&\in && 1-(1\pm \frac{\eps}{8}) \eta & (\text{Ineq.~\eqref{eq:gap2_5}}) \\
		&\in && (1\pm \frac{\eps}{8}) (1-\eta) & (\eta< 0.5)
		\end{align*}
		Similarly, we can prove that with probability at least $1-4e^{-\frac{\eps^2 \eta \lambda N}{192}}$, 
		\begin{itemize}
			\item $\pi_{01},\pi_{20}, \Pr\left[\widehat{Z}=1\mid f^\star=1,Z=0\right],$ $ \Pr\left[\widehat{Z}=0\mid f^\star=1,Z=1\right] \in (1\pm \frac{\eps}{8})\eta$;
			\item $\pi_{00},\pi_{11}, \Pr\left[\widehat{Z}=0\mid f^\star=1,Z=0\right],$ $ \Pr\left[\widehat{Z}=1\mid f^\star=1,Z=1\right] \in (1\pm \frac{\eps}{8})(1-\eta)$.
		\end{itemize}
		Applying these inequalities to Equations~\eqref{eq:gap2_1} and~\eqref{eq:gap2_2}, we have that with probability at least $1-4e^{-\frac{\eps^2 \eta \lambda N}{192}}$, 
		\begin{eqnarray*}
			\begin{split}
				\frac{\Pr\left[f^\star=1\mid \widehat{Z}=0\right]}{\Pr\left[f^\star=1\mid \widehat{Z}=1\right]} 
				& \in && (1\pm \eps)\cdot \frac{\eta\mu_0 + (1-\eta) (1-\mu_0)}{(1-\eta)\mu_0 + \eta (1-\mu_0)}\times \frac{(1-\eta)\mu_0 \gamma(f^\star, S)+\eta(1-\mu_0)}{\eta\mu_0 \gamma(f^\star,S)+(1-\eta)(1-\mu_0)} \\
				& \in && (1\pm \eps)\cdot \Gamma, 
			\end{split}
		\end{eqnarray*}
		and
		\[
		\frac{\Pr\left[f^\star=1\mid \widehat{Z}=1\right]}{\Pr\left[f^\star=1\mid \widehat{Z}=0\right]} \in (1\pm \eps)\cdot \frac{1}{\Gamma}.
		\]
		By the definition of $\gamma(f^\star, \widehat{S})$, we complete the proof.
	\end{proof}

\end{document}